\newcommand\numberthis{\addtocounter{equation}{1}\tag{\theequation}}

\newcommand{\lp}{\left(}
\newcommand{\rp}{\right)}

\newcommand{\mbf}{\mathbf}
\newcommand{\mc}{\mathcal}

\newcommand{\mbb}{\mathbb}
\newcommand{\bds}{\boldsymbol}

\newcommand{\os}[1]{\textcolor{red}{\textbf{OS: #1}}}
\newcommand\redout{\bgroup\markoverwith{\textcolor{red}{\rule[.5ex]{2pt}{0.4pt}}}\ULon}

\renewcommand{\P}{\mbb{P}}
\newcommand{\E}{\mbb{E}}

\newcommand*{\QEDB}{\null\nobreak\hfill\ensuremath{\square}}

\newcommand{\KL}[2]{\bds{D}_{\text{KL}}( #1 \| #2 )}

\newcommand{\dtv}{d_{\mathrm{TV}}}
\newcommand{\dflip}{d_{\mathrm{flip}}}

\newcommand{\eps}{\epsilon}

\newcommand{\Vast}{\bBigg@{5}}

\newcommand{\epochgap}{\Gamma_e}
\renewcommand{\Pr}{\P}

\documentclass[journal]{IEEEtran}
\ifCLASSINFOpdf
\else
\fi
\usepackage{hyperref}       
\usepackage{url}            
\usepackage{booktabs}       
\usepackage{amsfonts}       
\usepackage{nicefrac}       
\usepackage[dvipsnames]{xcolor}
\usepackage[normalem]{ulem}
\usepackage{microtype}      
\usepackage{lipsum}
\usepackage[utf8]{inputenc} 
\usepackage[T1]{fontenc}    
\usepackage{hyperref}       
\usepackage{url}            
\usepackage{booktabs}       
\usepackage{nicefrac}       
\usepackage{microtype}      
\usepackage{verbatim}
\usepackage{amsmath}
\usepackage{amsthm}
\usepackage{amssymb}
\usepackage{stackrel}
\usepackage{verbatim}
\usepackage{url}
\usepackage{hyperref}
\usepackage{graphicx}
\usepackage{enumerate}
\usepackage{algorithm}
\usepackage{algorithmicx}
\usepackage[noend]{algpseudocode}%

\usepackage{float}
\usepackage{blindtext}
\usepackage{multirow}
\usepackage{cite}
\usepackage{footnote}
\usepackage{enumitem}
\usepackage{subfigure}
\usepackage{verbatim}
\usepackage{dsfont}
\usepackage{amsmath}
\DeclareMathOperator*{\argmax}{arg\,max}

\usepackage{xcolor}

\setitemize{noitemsep}

\usepackage[export]{adjustbox}

\newtheorem{theorem}{Theorem}
\newtheorem{proposition}{Proposition}

\newtheorem{lemma}{Lemma}
\newtheorem{remark}{Remark}
\newtheorem{definition}{Definition}
\newtheorem{corollary}{Corollary}
\usepackage{mathtools} 

\DeclarePairedDelimiter\floor{\lfloor}{\rfloor}
\DeclarePairedDelimiter\ceil{\lceil}{\rceil}

\allowdisplaybreaks
\begin{document}
%
\title{Quantile Multi-Armed Bandits: Optimal Best-Arm Identification and a Differentially Private Scheme}
%
%
%

\author{Konstantinos E. Nikolakakis$^\dagger$, Dionysios S. Kalogerias$^\dagger$, Or Sheffet$^{**}$, Anand D.~Sarwate$^*$
\thanks{$^\dagger$K.E.~Nikolakakis and D.S. Kalogerias are with the Department of Electrical Engineering, Yale University, New Haven, CT, 06511 USA, (e-mail: \{konstantinos.nikolakakis, dionysis.kalogerias\}@yale.edu).}
\thanks{$^*$A.D. Sarwate is 
with the Department
of Electrical and Computer Engineering, Rutgers, The State University of New Jersey, Piscataway, NJ, 08854 USA (e-mail: anand.sarwate@rutgers.edu).}
\thanks{$^{**}$Or Sheffet is with the Faculty of Engineering, Bar-Ilan University, Ramat Gan, 5290002, Israel (e-mail: 
or.sheffet@biu.ac.il).}
\thanks{}
\thanks{}}

%
%

\markboth{}%
{}
%



\maketitle

\begin{abstract}
We study the best-arm identification problem in multi-armed bandits with stochastic rewards when the goal is to identify the arm with the highest quantile at a fixed, prescribed level. First, we propose a successive elimination algorithm for strictly optimal best-arm identification, show that it is $\delta$-PAC and characterize its sample complexity. Further, we provide a lower bound on the expected number of pulls, showing that the proposed algorithm is essentially optimal up to logarithmic factors. Both upper and lower complexity bounds depend on a special definition of the associated suboptimality gap, designed in particular for the quantile bandit problem~--- as we show, when the gap approaches zero, best-arm identification is impossible.
Second, motivated by applications where the rewards are private information, we provide a differentially private successive elimination algorithm whose sample complexity is finite even for distributions with infinite support and characterize its sample complexity. Our algorithms do not require prior knowledge of either the suboptimality gap or other statistical information related to the bandit problem at hand.\end{abstract}


%
\IEEEpeerreviewmaketitle

\section{Introduction}
%
%
%
%
\IEEEPARstart{M}{ulti}-armed bandits are an important class of online learning problems with a rich history (see the book by \cite{LSbandits_book:20} for a detailed treatment). In a \textit{stochastic} $K$-armed bandit problem, a learner is presented with a set of $K$ different actions (or \textit{arms}) $\{1, 2, \ldots, K\}$ and can sequentially take actions (\textit{pull arms}) to receive random rewards. The reward of arm $i$ at time $t$ is $X^i_t$. 
The learner may have one of a number of common objectives, such as to find the arm with the maximum $\mu_i=\E[X^i]$ to minimize cumulative regret~\cite{madani2004budgeted,bubeck2009pure}. 

In this paper, we study a different form of bandit problems in which the figure of merit is the \textit{left-side $q$-quantile} of the involved reward distributions, defined, for arm $i$, as $F_{i}^{-1}(q) = \inf \{ x : F_i(x) \ge q \}$, where $F_i(x) = \P[X^i \le x]$ is the corresponding cumulative distribution function (CDF) \cite{yu2013sample,Szorenyi15mab}. In particular, we study the problem of \emph{best-arm identification}, i.e., that of identifying the arm with the highest or lowest $q$-quantile, with as few samples as possible. 

The quantile bandit problem arises naturally in the context of \emph{risk-aware} optimization and learning, which has expanded considerably during the last decade~\cite{ruszczynski2006optimization,Sani12risk,shapiro2012minimax,ShapiroLectures_2ND,Tamar2017,Jiang2017,W.Huang2017,Kalogerias2018b,Vitt2018,Kalogerias2019,kagrecha2019distribution,pmlr-v89-cardoso19a,Kim2019,Zhou2020a}. There are many application scenarios which fit this quantile-based risk-aware setting:
    \begin{enumerate}[leftmargin=*]
        \item Arms are different feasible asset portfolio allocations~\cite{gaivoronski2005value} and the goal is to find the portfolio with the minimum potential monetary loss, within a target \textit{investment risk} $q$. 
        If such a (random) loss is denoted by $Z$, then this goal may be achieved by choosing $F_Z^{-1}(1-q)$ as the corresponding objective (to be minimized). In this context, the $(1-q)$-quantile is well-known as the \textit{Value-at-Risk at level $q$}, denoted as $\mathrm{V@R}_q(Z)\equiv F_Z^{-1}(1-q)$.
        \item Arms are different servers which can be assigned jobs and the rewards are delays. The goal is to identify the server with the highest 95th percentile delay because ``waiting for the slowest 5\% of the requests to complete is responsible for half of the total 99\%-percentile latency'' \cite{DeanB:2013tail}.
        \item Arms are different strains of an illness (e.g. different lung cancer genotypes) and the rewards are effectiveness of a proposed treatment on the strain. We wish to find the strain for which the treatment guarantees the highest effectiveness in at least $90\%$ of patients.
    \end{enumerate}

From a technical standpoint, the quantile bandit problem differs from the mean (or risk-neutral) bandit problem in a number of important ways. First, for the mean, the \emph{suboptimality gap} 
between the optimal $i^*$ and a suboptimal arm $i$ is simply $\mu_{i^*} - \mu_i$, whereas the absolute difference between quantiles is less useful. 
In fact, we show that the difference between the quantile of the optimal and a suboptimal arm can be arbitrarily small, while the hardness of best-quantile-arm identification remains insensitive. As we discuss below, the latter is captured by our gap definition for quantile bandits (Section~\ref{Gap}), which shows that the difficulty of the problem depends on the levels of the CDF in the neighborhood of the quantile rather than the actual values of quantile. Further, in contrast with the risk-neutral problem, the complexity of the quantile bandit problem is not affected by the tails or the range of the distributions' domain (support). Our analysis elucidates the above properties and provides guarantees on the proposed $\delta$-PAC (probably approximately correct) algorithm for highest quantile identification. In fact, we show that our new gap definition yields a fundamental quantity for the best-quantile identification problem, and the algorithm that we propose returns \textit{strictly} optimal solutions.

The advent of wide-scale data analytics has made privacy issues a growing concern. Differential privacy (DP)~\cite{DworkMNS06} has become the de-facto gold standard for privacy preserving data-analysis. For quantile bandit problems involving the data of individuals it is natural to model the reward information as private or sensitive. For example, the outcome of a treatment on an individual may be private information, but we would still like to find the most effective treatment. We therefore like to both identify the arm with the best quantile and protect the privacy of individuals. The goal is to minimize the ``cost of privacy'': how many \emph{more} samples does the private algorithm need over the non-private algorithm? To understand this, it is necessary to have a non-private baseline to measure against; our characterization of the sample complexity of the non-private problem establishes such a baseline.  We thus dedicate the first part of this work to showing that our notion of gap fully characterizes solvable instances. 




More specifically, in this paper we make the following contributions:
\vspace{-0.1cm}
\setitemize{leftmargin=10pt}
\begin{itemize}[leftmargin=*]
  \setlength\itemsep{0.1em}
  \item We provide a novel concentration bound for quantile estimates from $n$ i.i.d. samples, that applies for both continuous and discrete distributions. This result holds for any choice of $n$, thus it is useful for any generic sequential estimation procedure. Additionally, it efficiently captures the effect of quantile level values (close to zero $q\rightarrow 0$ and close to one $q\rightarrow 1$) and provides meaningful bounds for any $q\in (0,1)$, in contrast with alternative results in prior works that are uniform over $q$. Our concentration bound allows to solve the problem exactly and indicates that $\eps$-approximation bounds that appear in prior works are unnecessary.
  \item We provide a \emph{definition of the gap $\Delta_i$ at level $q$} between arm $i$ and the optimal arm that  generalizes those proposed in prior work \cite{Szorenyi15mab, david2016pure,HowardR19}.\footnote{A simultaneous (April 2021) preprint~\cite{HowardR19} considers a similar gap as ours using the lower quantile function (Ramdas, personal communication, 2020).} Our gap precisely captures the difficulty of the problem in the sense that when $\Delta_i = 0$ for all suboptimal arms $i$, no algorithm can hope to identify the arm with the higher $q$-quantile (Theorem \ref{lem:positive_gap}). The latter shows that our definition of gap provides a fundamental quantity for the best-quantile identification problem.
    \item We introduce a new \emph{pure-exploration successive elimination algorithm for quantile bandits} (Algorithm \ref{alg:succ_el_quant}), show that it is $\delta$-PAC (Theorem \ref{thm:identification}) and provide nearly matching upper (Theorem \ref{thm:upper_bound}) and lower (Theorem \ref{thm:lower_bound}) bounds on the sample complexity that depend on our improved gap definition. In fact, the upper bound (Theorem \ref{thm:upper_bound}) on the termination time of the algorithm is a high probability result, while the converse is a lower bound on the expectation of the termination time (Theorem \ref{thm:lower_bound}). 
    These results complement prior work
    on $\varepsilon$-optimal quantile bandits\footnote{An arm $i\in\mc{A}$ is $\varepsilon$-optimal at a level $q$ if and only if $F_i^{-1}(q+\varepsilon)\geq F_{i^*}^{-1}(q)$. For the definition and examples of the $\varepsilon$-optimality at a level $q$ see also~\cite[Definition 2]{Szorenyi15mab}.}~\cite{Szorenyi15mab, david2016pure,HowardR19} with $\delta$-PAC upper bounds and expectation lower bounds {(as we explain later, the difference in the type of upper and lower bounds is probably intrinsic to the problem setting under consideration)}. Our approach provides optimal solutions, and exact best arm identification for both continuous and discontinuous distributions in contrast with the $\varepsilon$-approximations of prior works, for which the algorithm does not terminate at levels of discontinuities when $\varepsilon\rightarrow 0$. Additionally, the approach of this work does not require any prior knowledge of either the suboptimality gap or other statistical information related to the bandit problem at hand. On the other hand, for $\varepsilon$-approximations to achieve the best approximation the value of $\varepsilon$ has to be chosen smaller than the value of gap, which is not known beforehand.  
    
    %
    \item Using our modified confidence intervals, we propose the first \emph{differentially private best-arm identification algorithm for quantile bandits} (Algorithm \ref{alg:DP_SE_Quantiles}), prove that it is private (Theorem \ref{thm:Alg_is_DP}), and analyze the trade-off between privacy budget and sample complexity (Theorems \ref{thm:utility_DP_quantile_SE} and \ref{thm:num_pulls_DP_alg}). Interestingly, the sample complexity bound for our private algorithm has no dependency on the support size of the distribution, which is \emph{necessary} in the case where one wishes to privately estimate the $q$-quantile~\cite{BeimelNS13, FeldmanX14, BunNSV15} rather than identify which arm has highest quantile. This difference between estimation and identification may be of interest for future private algorithms.
\end{itemize}


\subsection{Prior Work}
Most works on bandit problems under stochastic rewards considers the problem of best-arm identification for the mean. This setting received renewed attention after the work of Even-Dar, Mannor, and Mansour~\cite{even2002pac} on the MAB problem in the PAC learning setting. Later work follows by considering extensions/variations of this problem~\cite{kalyanakrishnan2012pac,gabillon2012best,karnin2013almost,jamieson2013finding,jamieson2014lil}. 
Lower bounds on the sample complexity in terms of the mean suboptimiality gap were proved by Mannor and Tsitsiklis~\cite{mannor2004sample}, and Anthony and Bartlett~\cite{anthony2009neural}. Alternative lower bounds also include results based on the KL-divergence of the arms' distributions~\cite{burnetas1996optimal,chen2015optimal,kaufmann2016complexity,garivier2016optimal}. Cappé et al.~\cite{cappe2013kullback} present the KL-UCB algorithm that achieves (asymptotically) optimal sample complexity rates by matching known lower bounds. In parallel, prior works encompass non-stochastic approaches~\cite{jamieson2016non,li2016novel}, as well.

Bandit models with non-stationary~\cite{allesiardo2017selection,allesiardo2017non}, or heavy-tailed~\cite{bubeck2012regret,bubeck2013bandits} distributions are most related to this work, since the quantile problem is often of interest in these settings. Kagrecha et al.~\cite{kagrecha2019distribution} consider the unbounded reward best-arm identification problem 
while variants of regret-based approaches include minimization of generalized loss functions~\cite{li2018bandit,berthet2017fast,boda2019correlated,maillard2013robust}. 
More recent works also consider risk measures, for instance conditional value-at-risk (CVaR)~\cite{yu2013sample,pmlr-v89-cardoso19a}, mean-variance~\cite{even2006risk,Sani12risk,vakili2016risk} 
 or unified approaches~\cite{cassel2018general}. 
These are complemented by concentration results on risk measure estimators~\cite{wang2010deviation,kolla2019concentration,bhat2019concentration}.

Our results are closely related to prior work on quantile bandit problem for best-arm identification~\cite{yu2013sample,Szorenyi15mab,david2016pure,pmlr-v101-torossian19a,HowardR19}. Altschuler et al.~\cite{altschuler2019best} specifically study \emph{median} identification for contaminated distributions in the robust statistics sense. 
Of these, the  most highly related works are the beautiful work by Sz\"{o}r\'{e}nyi et al.~\cite{Szorenyi15mab}, the refinement by David and Shimkin~\cite{david2016pure}, and the preprint of Howard and Ramdas~\cite{HowardR19}. Our algorithm uses successive elimination (similarly to~\cite{Szorenyi15mab}), while Howard and Ramdas~\cite{HowardR19} consider the UCB approach. David and Shimkin~\cite{david2016pure} and Howard and Ramdas~\cite{HowardR19} tighten the upper bounds to a double-logarithmic factor. The epoch-based algorithms provide asymptotically tighter sample complexity bounds at the expense of a much larger constant. We therefore present both versions of the successive elimination algorithm for best quantile identification; the standard and the epoch-based approach. Our results for the quantile bandits problem complement the prior works by Sz\"{o}r\'{e}nyi et al.~\cite{Szorenyi15mab} and  David and Shimkin~\cite{david2016pure} by solving the problem of exact arm identification, rather than providing  an approximation. A discussion about exact and approximate approaches follows.

\subsubsection{Comparison with $\varepsilon$-approximate approaches}
The aforementioned works~\cite{Szorenyi15mab,david2016pure} study $\varepsilon$-approximate best-arm identification: the algorithm returns an arm which is within $\varepsilon$ of optimal, for some $\varepsilon\geq 0$. First we discuss the major differences on the approach, algorithm and theoretical guarantees in this work and those in prior works. Then we continue by stating advantages and disadvantages between our approach and $\varepsilon-$approximations. To begin with, the neat algorithm and analysis by Sz\"{o}r\'{e}nyi et al.~\cite{Szorenyi15mab} solves the problem of quantile bandits in a variety of cases. These cases include continuous and discrete distributions for $\varepsilon>0$ and continuous distributions for $\varepsilon=0$. That is, the algorithm by Sz\"{o}r\'{e}nyi et al.~\cite{Szorenyi15mab} does not terminate in case of discontinuous distribution for $\varepsilon=0$ when we are interested at the level of discontinuity. This fact can be verified theoretically and experimentally. Theorem 1 by  Sz\"{o}r\'{e}nyi et al.~\cite{Szorenyi15mab} (for $\varepsilon=0$) involves a gap whose definition is slightly different than our gap. In many cases of discrete distributions the gap by  Sz\"{o}r\'{e}nyi et al.~\cite{Szorenyi15mab} is zero, while the gap in this work is positive, showing that the problem instance is not hard. Thus the algorithm of this paper terminates at levels of quantiles with discontinuity as long as the problem is feasible.

To understand this further, we explain a key difference between the algorithm by Sz\"{o}r\'{e}nyi et al.~\cite{Szorenyi15mab} and the algorithm of this work. As we discuss below, this difference is also crucial for the performance of the two algorithms. The decision rule in Algorithm 1 (lines 9 and 11) by Sz\"{o}r\'{e}nyi et al.~\cite{Szorenyi15mab} involves different statistics than those that we consider. Specifically, to characterize the setting of $\varepsilon=0$ for discontinuous distributions both $\hat{Q}^{}_{t,X^i}\triangleq \inf \{\xi: \hat{\P}_t [X^i\leq \xi ]\geq q  \}$ and $\hat{Q}^{-}_{t,X^i}\triangleq \inf \{\xi: \hat{\P}_t [X^i\leq \xi ]> q  \}$ are required (see Algorithm \ref{alg:succ_el_quant} line 13), while their work~\cite[Algorithm 1]{Szorenyi15mab} involves only the quantity $\hat{Q}^{}_{t,X^i}$. That difference together with the "less or equal than" (current work) instead of a strict inequality (prior work~\cite{Szorenyi15mab}) in the elimination step, are sufficient to make the algorithm terminate for cases of $\varepsilon=0$ and discontinuous distributions. As a consequence of the different statistics involved, the proof of the concentration bound is also different. Our approach is based on Hoeffding's inequality, while the proof by Sz\"{o}r\'{e}nyi et al.~\cite{Szorenyi15mab} considers Massart's DKW inequality. Notice that the latter of the two approaches does not directly provide a concentration bound for the statistic $\hat{Q}^{-}_{t,X^i}$, however the Chernoff-Hoeffding bound solves the problem in the expense of a larger constant. For an alternative approach of the concentration bound proof that involves smaller constants and uniformity over all quantiles see also the preprint by Howard and Ramdas~\cite{HowardR19}.

Sz\"{o}r\'{e}nyi et al.~\cite{Szorenyi15mab} use a gap which depends on the parameter $\varepsilon$, while other prior works~\cite{david2016pure,HowardR19} provide an alternative gap that does not involve the quantity  $\varepsilon$, but study $\varepsilon$-approximate algorithms. By contrast, our algorithm returns the optimal arm, and we show that when our gap is $0$ then a suboptimal distribution (with small $q$-quantile) is \emph{actually indistinguishable} from a distribution with a larger $q$-quantile (see Section~\ref{Gap}, Theorem \ref{lem:positive_gap}). The main advantage of approaches that consider $\varepsilon-$approximations~\cite{Szorenyi15mab,david2016pure,HowardR19} is that the algorithm terminates even when the gap is zero (by breaking ties arbitrarily), at the expense of approximating the quantile estimate ($\varepsilon>0$). Still, in applications we may not always be able to choose $\varepsilon$ to achieve best approximation unless there is side information for the distributions of the data. For instance, if the value of the gap is not known beforehand, we may accidentally choose $\varepsilon$ to be much greater than the gap and the output of the algorithm can possibly crudely approximate the solution of the problem by returning a rather suboptimal arm. In contrast, the algorithms for exact best arm identification of this work do not require prior knowledge of any side information, making them of interest for these applications. Additionally, our results substantially differ from those by David and Shimkin~\cite{david2016pure}. Specifically, their Theorem 1 considers the case of continuous distributions, while we provide unified analysis for both discrete and continuous distributions. Further results~\cite[Theorem 2 and Theorem 3]{david2016pure} show that the algorithm by David and Shimkin is not guaranteed to terminate (unbounded expected number of samples) when $\varepsilon\rightarrow 0$. In contrast, the current work solves the problem of exact estimation.

Lastly, in our approach the upper ($\delta$-PAC) bound does not appear to provide an upper bound for the expected number of pulls, and the lower bound on the expected number of pulls does not directly guarantee a lower $\delta$-PAC bound. Specifically, under the low (at most $\delta$) probability event, there exist instances for which the problem reduces to that of zero gap problem. If the (unique) optimal arm is mistakenly eliminated under the low probability event, while at least two distributions of the remaining sub-optimal arms are identical, then 
the algorithm 
does not terminate, because the gap restricted to the remaining identical arms is zero. In fact, no algorithm can identify the best-quantile arm under the zero gap case (see Section~\ref{Gap}, Theorem \ref{lem:positive_gap}); however, $\varepsilon$-approximate solutions (with $\varepsilon>0$) terminate by breaking ties arbitrarily.

\subsubsection{Prior work on Differential Privacy.} The field of differentially private machine learning is, by now, too large to summarize here, as the following (non-exhaustive) list of works discussing learning quantiles/threshold-functions attests~\cite{NissimRS07, ChaudhuriH11, BeimelNS13, BeimelNS13b, FeldmanX14, BunNSV15, AlonLMM19, KaplanLMNS20}. For differentially private multi-armed bandit problems for the mean, 
Mishra and Thakurta~\cite{MishraT15} were the first to analyze a differentially private (DP) algorithm for multi-armed bandit, building a private variant of the UCB-algorithm~\cite{Auer2002} using the tree-based algorithm~\cite{ChanSS10,DworkNPR10}. Shariff and Sheffet~\cite{ShariffS18} have proven that any $\eps$-DP algorithm (see Section \ref{sec:dp-best-quantile}, Definition \ref{def:dp_BAI_alg}) for the (mean) multi-armed bandit problem must pull each suboptimal arm $i$ at least $\Omega\left(\log(T)/\epsilon(\mu_{i^*}-\mu_{i})\right)$ many times (with $i^*$ denoting the optimal arm, of largest mean-reward $\mu_{i^*} = \max_{i\in \mc{A}}\mu_i$) which doesn't quite meet the DP-UCB algorithm's upper bound. 
Most recently Sajed and Sheffet~\cite{SajedS19} gave a DP version of successive elimination whose regret matches the lower bound~\cite{ShariffS18}. 

\section{Problem Statement}


We consider a $K$-armed unstructured stochastic bandit $\nu = (\nu_i : i \in \mc{A})$, where $\mc{A}\triangleq\{1,2,\ldots,K\}$ is the set of arms and $\nu_i$ are probability measures.
For the $i$-th arm, let $X^i$ be a random variable with distribution $\nu_i$. We will describe distributions by their cumulative distribution functions (CDFs). 

\begin{definition} \label{def:quantile}
Let $F_i(\cdot)$ be the CDF of $X^i$ for arm $i$.
The \emph{$q$-quantile} $F_i^{-1}(q)$ is defined as 
\begin{align}
    F_{i}^{-1}(q)&\triangleq \inf \{\xi: \P[X^i\leq \xi ]\geq q  \}
    \label{eq:quantile_definition}
\end{align}  
and the \emph{best arm} is defined as
    \begin{align}
    i^* \triangleq \argmax_{i\in \mc{A}} F^{-1}_i (q).
    \end{align} 
\end{definition}

For simplicity, we assume that the best arm is unique in the set $\mc{A}$. We denote the set of suboptimal arms as $\mc{A}^{-i^*}\triangleq \{1,2,\ldots,K\}\setminus \{i^*\}$. Given $n$ samples the estimated CDF of $X^i$ is $\hat{F}_{n,i} (x)\triangleq\frac{1}{n}\sum^{n}_{\ell=1} \mbb{I} \{X^i_\ell\leq x \}$. We denote the set of samples from arm $k$ as $\{X_i^k\}^n_{i=1}$, while 
for the $j^\text{th}$ order statistic of $\{X_i^k\}^n_{i=1}$ we use the standard notation $X_{(j)}^{k}$ for $j\in \{1,2,\ldots n\}$. If the value $j$ in $X_{(j)}^{k}$ appears out of the range $(1,n)$, while $n$ samples are available, then it is considered equal to the \textit{closest of the two values $1$ or $n$}. 

An algorithm for our quantile bandit chooses at each time $n$ an arm $i_n \in \mc{A}$ and obtains a reward $X^{i_n}_n \sim \nu_{i_n}$. The algorithm terminates by stopping sampling and declaring an arm $\hat{k}$ as the arm with the highest $q$-quantile, and succeeds if actually $\hat{k} = i^*$. We call an algorithm $\delta$-PAC if $\P(\hat{k} = i^*) \ge 1 - \delta$.

\section{Concentration Bound}
We proceed by providing a concentration bound for quantile estimation that applies for both discrete and continuous distributions. 
\begin{theorem}[Concentration Bound] \label{thm:GCB}
Choose a level $q\in (0,1)$. Fix $\delta\in (0,1)$. For any $n\in\mbb{N}$, if \begin{align}
    \sqrt{\frac{  \log(2/\delta)}{2n}} \leq \zeta \leq \min\{q,1-q\}  \label{eq:inequality_conc_1} \end{align} then \begin{align}
   \P\lp F_X^{-1} \lp q\rp\notin \left[X_{(\floor*{n(q- \zeta)})} , X_{(\ceil*{n(q+\zeta)})}  \right] \rp\leq \delta.
\end{align}
\end{theorem}\noindent In contrast, with concentration bounds that are uniform over the values of the level $q$~\cite{Szorenyi15mab}, Theorem \eqref{thm:GCB} shows the dependence of the required number of samples $n$ with respect to $q$ through the inequality \eqref{eq:inequality_conc_1}. This property explicitly expresses the difficulty of the problem when estimating the quantile close to the tails of the distribution. We provide the proof of Theorem \ref{thm:GCB} in Appendix \ref{app:properties}. 


\section{Optimal Best-Arm Identification\\for Quantile Bandits}

\subsection{(Non-Private) Successive Elimination Algorithm}
\label{sec:succ_el_algorithm}
We choose to study successive elimination (SE) rather than a variant of UCB~\cite{Auer2002} (adopted by Howard and Ramdas~\cite{HowardR19} for quantiles) for the following reasons. Firstly, we prove matching upper and lower bounds on the sample complexity, showing our SE algorithm is essentially optimal (up to logarithmic terms). Secondly, since we are also interested in developing differentially private algorithms (see Section \ref{sec:dp-best-quantile}), the SE algorithm is more ``privacy friendly'', because the sampling strategy is independent of the data and it uses confidence bounds in terms of the order statistics. Finally, there is no private analog to UCB when the distributions have infinite support.
\begin{algorithm}
    \begin{algorithmic}[1]
    \caption{Successive Elimination for Quantiles (SEQ)} 
	\label{alg:succ_el_quant}
	\Require $\delta, q$
	\State $\mc{A}\leftarrow\{1,\ldots,K \}$ 
	\State $\mathrm{D}(n) = \sqrt{\frac{  \log(4Kn^2/\delta)}{2n}}$
	\State Find the first $n_*\in\mbb{N}\setminus\{1\}$ such that $\mathrm{D}(n)\leq \min\{q,1-q\}$
	    \State $n\leftarrow n_*$
	    \State Pull $n$ times each arm $k\in\mc{A}$, obtain new samples $X^k_{1},\ldots,X^k_{n}$ for all $k\in K$
	    \While{$|\mc{A}|>1$}
	    \State Increment $n\leftarrow n+1$
	    \State Set $\mathrm{D} \leftarrow  \sqrt{\frac{  \log(4Kn^2/\delta)}{2n}}$
	    \State Pull each arm in $\mc{A}$, obtain samples $X^k_n$ for all $k\in \mc{A}$ 
	    \State Update $X^k_{(\floor*{n(q-\mathrm{D})})}$ and  $X^k_{(\ceil*{n(q+\mathrm{D})})} $ for all $k\in \mc{A}$  
	    \vspace{+0.1cm}
	    \For {{\bf each} pair $(j,i)\in \mc{A}\times \mc{A}$ such that $j\neq i$}
	    \vspace{+0.05cm}
	     \If { $X^j_{(\floor*{n(q-\mathrm{D})})}  \geq X^i_{(\ceil*{n(q+\mathrm{D})})}$  } $\mc{A}\leftarrow \mc{A}\setminus \{i\}$
	    \EndIf
	    \EndFor
	    \EndWhile
       \State \Return $\mc{A}$    
      \end{algorithmic}
\end{algorithm}

Our Successive Elimination algorithm for Quantiles (SEQ) Algorithm is shown in Algorithm \ref{alg:succ_el_quant}. To explain SEQ (Algorithm \ref{alg:succ_el_quant}), we define the sequence $\mathrm{D}(n)\triangleq\sqrt{  \log(4Kn^2/\delta)/2n}$, (we denote $\mathrm{D}(n)$ as $\mathrm{D}$ for sake of space) and we use a concentration bound on the quantile (see Lemma \ref{prop:E^c}) \begin{align}\label{eq:concentration_bound_main}
  \!\! \P\big( F_i^{-1} \lp q\rp\in \big[X^i_{(\floor*{n(q-\mathrm{D})})} , X^i_{(\ceil*{n(q+\mathrm{D})})}  \big] \big) > 1-\frac{\delta}{2Kn^2}.
\end{align} 
The latter yields the elimination condition in line 13 of Algorithm \ref{alg:succ_el_quant}. Specifically, when the inequality $X^j_{(\floor*{n(q-\mathrm{D})})}  \geq X^i_{(\ceil*{n(q+\mathrm{D})})}$ holds then $F_{i}^{-1}(q)\leq F_{j}^{-1}(q)$ \vspace{0.25mm} with probability at least $1-\delta$ (by applying union bound in \eqref{eq:concentration_bound_main} over all times $n$ and arms $K$). Thus to identify the arm with the maximum quantile, whenever $X^j_{(\floor*{n(q-\mathrm{D})})}  \geq X^i_{(\ceil*{n(q+\mathrm{D})})}$, we remove $i$ from $\mc{A}$.\footnote{To identify the arm with the minimum quantile, we modify line 13 of the algorithm as follows: If $X^j_{(\floor*{n(q-\mathrm{D})})}  \geq X^i_{(\ceil*{n(q+\mathrm{D})})}$, then we remove $j$ from $\mc{A}$.} 

A variant of the algorithm would be to take samples in epochs of increasing size. We consider this approach in the development of the differentially private version of Algorithm \ref{alg:succ_el_quant}, which reduces to a non-private epoch-based variant of Algorithm \ref{alg:succ_el_quant} (Section \ref{sec:dp-best-quantile}, Algorithm \ref{alg:DP_SE_Quantiles}). This epoch-based algorithm improves the logarithmic (and inconsequential) part of the bound of Theorem~\ref{thm:upper_bound} from $\log(1/\Delta_i)$ to $\log\log(1/\Delta_i)$ and matches asymptotically the bound for UCB~\cite{HowardR19} (see the discussion at the end of Section~\ref{sec:dp-best-quantile}).


\subsection{Suboptimality gap}\label{Gap}

We first define the suboptimality gap between the best arm $i^*$ and any suboptimal arm.

 
\begin{definition}\label{dap_definition}
The \emph{suboptimality gap} $\Delta_{i}$ (also denoted as $\Delta(F_i, F_{i^*})$) between the optimal arm $i^*$ and any suboptimal arm $i$ at level $q\in (0,1)$ is
\begin{align}
\label{eq:gap_distributions}
    \Delta_{i} \triangleq  \sup  \{\eta\geq 0:~ F_{i}^{-1}(q+\eta) \leq F_{i^*}^{-1}(q-\eta) \}.
\end{align}
\end{definition}

How can we interpret this gap? Roughly speaking, it is the amount of probability mass needed to swap the order of the quantiles. To get further insight into the definition \eqref{eq:gap_distributions}, notice that $\mathrm{D}(n)$ is decreasing with respect to $n$, and the elimination occurs at the first time (maximum value of $\mathrm{D}$) that gives $X^j_{(\floor*{n(q-\mathrm{D})})}  \geq X^i_{(\ceil*{n(q+\mathrm{D})})}$.\vspace{0.25mm} In fact, the value $\Delta_i$ in \eqref{eq:gap_distributions} acts as a threshold on the quantity $\mathrm{D}$ in the analysis of the algorithm (proof of Theorem \ref{thm:upper_bound}). Our definition of gap applies on continuous, discrete, and mixture distributions. 

Most importantly, the key point in the Definition \ref{dap_definition} of the quantile suboptimality-gap is that it \emph{fully characterizes} the pairs of distributions for which we can discern that one has a higher $q$-quantile than another \emph{from any number of samples}. Formally, for a pair of distributions $(F_{\rm l} , F_{\rm h} )$ where the former has a suboptimal $q$-quantile than the latter, namely $F_{\rm l} ^{-1}(q)\leq F_{\rm h} ^{-1}(q)$, we define the \emph{distance to quantile-flip at $q$} as
\begin{align}
    &\!\!\!\!\dflip(F_{\rm l} , F_{\rm h} )\label{eq:dist_to_quantile_flip} \\&\!\!\!\!\!\triangleq\! \inf_{( G_{\rm h}  ,G_{\rm l} ): G_{\rm h} ^{-1}(q) > G_{\rm l} ^{-1}(q)} \max \{ \dtv(F_{\rm l} , G_{\rm h} ), \dtv(F_{\rm h} , G_{\rm l} )  \},\nonumber
\end{align} 
and by $\dtv(\cdot , \cdot )$ we denote the total variation. Next we provide a rigorous result which shows that the gap $\Delta_{i}$ is indeed a fundamental quantity which characterizes the complexity of the best quantile identification problem.

\begin{theorem}\label{lem:positive_gap}
For any $0<q<1$ and any two distributions $F_{i}$ and $F_{i^*}$ such that $F_{i}^{-1}(q) \leq F_{i^*}^{-1}(q)$ it holds that $\dflip(F_{i}, F_{i^*})=\Delta(F_i, F_{i^*})$ provided that $\Delta(F_i, F_{i^*})<\min\{q,1-q\}/2$.
\end{theorem}

We provide the proof of Theorem \ref{lem:positive_gap} in Appendix \ref{app_quantile_flip}. Theorem \ref{lem:positive_gap} shows that if $\Delta_i=0$ then $ \dflip(F_{\rm l} , F_{\rm h} )=0$ and \emph{no algorithm} can distinguish which arm has the higher $q$-quantile, regardless of its sample-size: every batch of samples can be generated by a quantile flip pair $( G_{\rm h}  ,G_{\rm l} )$ with the same probability.
Conversely, when $\Delta_i>0$ we devise an algorithm that discerns which arm has the higher $q$-quantile using $\tilde O(\Delta_i^{-2})$ many examples\footnote{The notation $\tilde O (\cdot)$ denotes order up to logarithmic factors.} from each arm and argue that this bound is optimal in the sense that there exists a collection of $K$ distributions requiring $\tilde O(\Delta_i^{-2})$ many examples from each distribution (Section~\ref{subsec:analysis_succ_el_quant}). Corollary \ref{corollary_cases} in Appendix \ref{appendix_Gap} provides the cases for which $\Delta>0$ or $\Delta=0$. We continue by providing graphical representations and properties of the gap in certain cases. Finally, we present the main differences between our definition and definitions in prior work.

\subsubsection{Graphical illustration of the gap} Figure \ref{fig:gaussian_exponential_gap} shows the gap as a function of the level $q$ of the quantile for two continuous distributions, the Gaussian and the exponential. We vary the optimal distribution by altering the parameter (variance or rate). For the Gaussian example (left) we look at the gap between a (suboptimal) $\mc{N}(0,2)$ distribution and Gaussians of higher variance. As expected, when looking at the median the gap is $0$ since they are both symmetric distributions. More interestingly, the best-arm identification problem becomes easiest when looking at some quantile $q^*$ (or $1 - q^*$) that lies between $1/2$ and $1$. The problem becomes hard again when looking at the tails of the distribution. For the exponential distribution we compare to a rate $\lambda_i = 1$ for smaller values of the rate. As the difference in rates grows, the problem becomes easier, as expected. Here too we see an optimal $q^*$ between $1/2$ and $1$ for which the top quantile is easiest to identify.
While analytical expressions for these optimal points could possibly be derived through analyzing the corresponding densities, this is not the focus of our work.

\begin{figure}[t!]
    \centering
    \includegraphics[width=0.45\textwidth, valign=t]{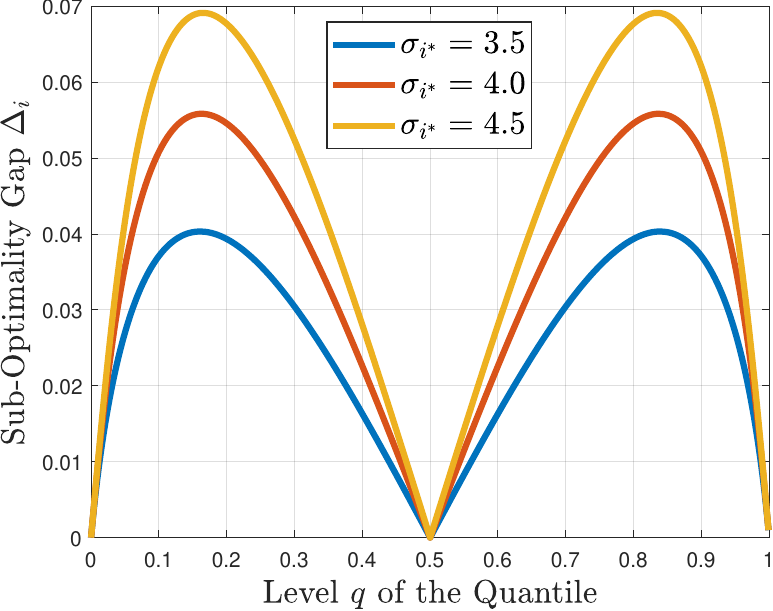}
    \vspace{8bp}
    \\
    \includegraphics[width=0.45\textwidth, valign=t]{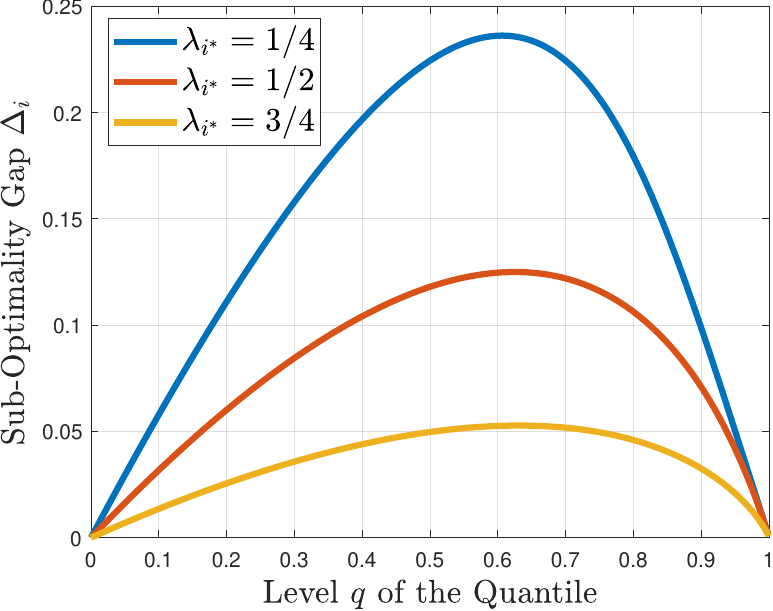}
    \caption{Illustration of the suboptimality gap for two example distributions. Upper: Gaussian with $\mu_{i^*}=\mu_{i}=0$, $\sigma_i=2$ and different values of $\sigma_{i^*}$, Lower: Exponential with $\lambda_i=1$.}
    \label{fig:gaussian_exponential_gap}
\end{figure}

For discrete distributions, we can show that the difference between the quantiles can become arbitrarily small while the definition of the gap and the sample complexity of Algorithm \ref{alg:succ_el_quant} remain insensitive, see Figure \ref{fig:quantile_diff}. Additionally, while the difference between the quantile values is not the correct definition to use for the gap in general, the two quantities are related in the case of Lipschitz CDFs.

\begin{proposition} \label{prop:lipschitz_cdf}
Suppose $F$ and $G$ are two distributions with $L$-Lipschitz continuous and strictly increasing CDFs. Then the following inequality holds $\Delta(F,G) \le \frac{L}{2} |F^{-1}(q) - G^{-1}(q)|$.
\end{proposition}
\begin{proof}[Proof of Proposition \ref{prop:lipschitz_cdf}]
By definition, we have
    \begin{align*}
    \eta &= \left|F\lp F^{-1}(q + \eta)\rp - F\lp F^{-1}(q) \rp\right|\\
    &\le L \left| F^{-1}(q + \eta) - F^{-1}(q) \right|
    \!= L \lp F^{-1}(q + \eta) - F^{-1}(q) \rp\\
    \eta &= \left|G\lp G^{-1}(q)\rp - G\lp, G^{-1}(q-\eta) \rp\right|\\
    &\le L \left| G^{-1}(q) - G^{-1}(q-\eta) \right| 
    \!= L \lp G^{-1}(q) - G^{-1}(q-\eta) \rp\!.
    \end{align*}
So,
    \begin{align*}
    \!2 \eta + L\! \lp G^{-1}(q-\eta) - F^{-1}(q + \eta) \rp
    \!\le L\! \lp G^{-1}(q) - F^{-1}(q) \rp\!.
    \end{align*}
From the definition of the gap, taking the supremum over $\eta$ gives
    \begin{align*}
    \Delta(F,G) \le \frac{L}{2} \left| F^{-1}(q) - G^{-1}(q) \right|.
    \end{align*} This completes the proof.
\end{proof} We continue by providing the difference between our definition for the gap in comparison with a similar definition in prior work. By providing a simple example, we explain that previous definitions fail to capture certain cases of interest.

\subsubsection{Difference between the proposed gap and prior work}\label{sec:Difference between the proposed gap and prior work} Although the suboptimality gap we propose (Definition \ref{dap_definition}) may look similar to those proposed in prior works~\cite{Szorenyi15mab,david2016pure}, there are several important differences. Earlier work by Sz\"{o}r\'{e}nyi et al.~\cite{Szorenyi15mab} explicitly incorporates the approximation parameter, whereas our definition depends only on the arm distributions. Further, the proposed gap differs from that in~\cite{Szorenyi15mab} when $\varepsilon=0$, because ''less or equal than'' takes the place of a strict inequality. This is not a trivial point, because for the case of discrete distributions the gap by Sz\"{o}r\'{e}nyi et al. can be zero, while the gap of the present work is positive, Algorithm \ref{alg:succ_el_quant} terminates and the problem is not hard. For instance, consider an example of two arms with the sub-optimal arm following a Bernoulli distribution with probability $1/2$ and support $\{1,2\}$, and with the optimal arm taking the value $2$ with probability $1$. Then the problem is not hard in terms of the sample complexity; Definition \ref{dap_definition} gives $\Delta_i>0$, while the gap in Sz\"{o}r\'{e}nyi et al.~\cite{Szorenyi15mab} is zero when $\varepsilon=0$. Additionally, in Theorem \ref{lem:positive_gap} we show that if $\Delta_i=0$, no algorithm can identify the best quantile-arm with probability greater than $1/2$. Further, we provide a minimax lower bound on the expected number of pulls based on the gap $\Delta_i$ (Theorem \ref{thm:lower_bound}); the latter shows that our upper bound is optimal up to logarithmic factors. Finally, the gap in \cite{david2016pure} involves a strict inequality instead of ''less or equal than'' and the supremum involves the distribution of the sub-optimal arm. This definition does not capture the difficulty of the problem for general cases of discrete distributions as we discussed above. The gap in this work (see Definition \ref{dap_definition}) captures the difficulty of the problem for discrete, continuous distributions, as well as for mixtures.

\DeclareRobustCommand{\LemmaDistFlip}{For any $0<q<1$ and any two distributions $F_{i}$ and $F_{i^*}$ such that $F_{i}^{-1}(q) \leq F_{i^*}^{-1}(q)$ it holds that $\dflip(F_{i}, F_{i^*})=\Delta(F_i, F_{i^*})$ provided that $\Delta(F_i, F_{i^*})<\min\{q,1-q\}/2$.}
\begin{figure}[t]
    \centering
    \!\!\!\!\includegraphics[width=0.45\textwidth, valign=c]{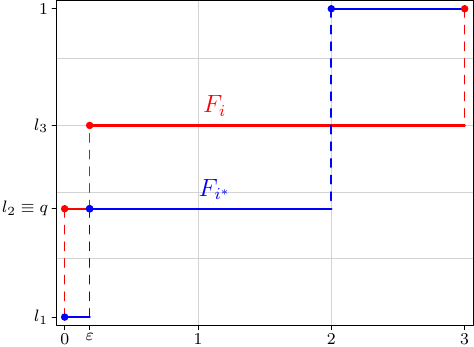}
    \hspace{3.1pt}
    \\
    \vspace{8bp}
    \includegraphics[width=0.455\textwidth, valign=c]{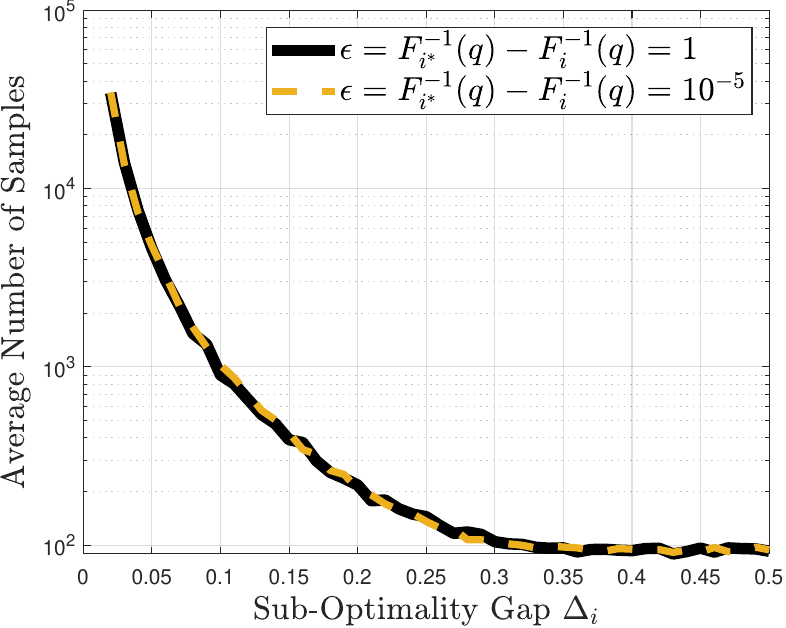}
    \caption{Upper: Distributions $F_{i^*}(\cdot),F_i(\cdot)$ of the optimal and suboptimal arm. Note that $0=F^{-1}_{i}(q)<F^{-1}_{i^*}(q)=\eps$, for some $\eps\in (0,1]$. However, the gap is $\Delta_i = \min\{ l_3 -q , q-l_1 \}$, independent of $\eps$. Lower: Experimental evaluation of the average number of samples at termination of Algorithm 1 for $\eps=1$ and $\eps=10^{-5}$.}
    \label{fig:quantile_diff}
\end{figure}


\subsection{Analysis}
\label{subsec:analysis_succ_el_quant}

Our first result guarantees that Algorithm \ref{alg:succ_el_quant} eliminates the suboptimal arms while the unique best arm remains in the set $\mc{A}$ with high probability until the algorithm terminates. For the rest of the paper we assume that $\Delta_i>0$ for all $i\in \{1,2\ldots,K\}\setminus i^*$. 

\begin{theorem}\label{thm:identification}
Algorithm \ref{alg:succ_el_quant} is $\delta$-PAC.
\end{theorem}\noindent To prove Theorem \ref{thm:identification}, recall that $n^*$ is the smallest integer that satisfies the inequality $\mathrm{D} (n^*)\leq \min\{q,1-q\}$. First, we show that the event $\mc{E}$ defined as \begin{align*}
    \mc{E}\triangleq \bigcap^{K}_{k=1} \bigcap^{\infty}_{n=n^*} \left\{  F_{k}^{-1}(q)\in \left[X^k_{(\floor*{n(q-\mathrm{D}(n))})} , X^k_{(\ceil*{n(q+\mathrm{D}(n))})}  \right] \right\} 
\end{align*} occurs with probability at least $1-\delta$. 

\begin{lemma} \label{prop:E^c}
Choose a level $q\in (0,1)$ and fix $\delta\in (0,1)$. 
Then $\P \lp \mc{E}^c \rp \leq \delta$.
\end{lemma}

\begin{proof}
For every $n$ such that \begin{align}
    \mathrm{D}(n)=  \sqrt{\frac{  \log(4Kn^2/\delta)}{2n}}\leq \min\{q,1-q\},
\end{align} from Lemma \ref{thm:GCB} in Appendix \ref{app:properties}, we get  \begin{align}\nonumber
    &\P \lp \left\{  F_{k}^{-1}(q)\notin \left[X^k_{(\floor*{n(q-\mathrm{D}(n))})} , X^k_{(\ceil*{n(q+\mathrm{D}(n))})}  \right] \right\} \rp\\&\leq \frac{\delta}{2Kn^2},
\end{align} and $k\in\{1,2,\ldots,K\}$. We conclude that
\begin{align*}
    &\P \lp \mc{E}^c \rp\\
    &\leq \sum^{K}_{k=1} \sum^{\infty}_{n=n^*} \!\P \lp    F_{k}^{-1}(q)\notin \left[X^k_{(\floor*{n(q-\mathrm{D}(n))})} , X^k_{(\ceil*{n(q+\mathrm{D}(n))})}  \right]  \rp\\
    &\leq  \sum^{K}_{k=1} \sum^{\infty}_{n=n^*} \frac{\delta}{2Kn^2}\\
    & =\sum^{\infty}_{n=n^*} \frac{\delta}{2n^2}\\
    & \leq \delta \sum^{\infty}_{n=1} \frac{1}{2n^2}\\
    & \leq \delta.\numberthis
\end{align*}
This proves the result stated in the lemma.
\end{proof}

\noindent We continue by providing the proof of Theorem \ref{thm:identification}.

\begin{proof}[Proof of Theorem \ref{thm:identification}]
Lemma \ref{prop:E^c} gives $\P (\mc{E})>1-\delta$. Under the event $\mc{E}$ the following inequalities hold
\begin{align}\label{eq:pairs_of_inequalities}
   \!\!\!\! F_j^{-1}(q) \geq X^j_{(\floor*{n(q-\mathrm{D}(n))})} \text{ and } F_i^{-1}(q) \leq X^i_{(\ceil*{n(q+\mathrm{D}(n))})}.
\end{align} Every time that the stopping condition $X^j_{(\floor*{n(q-\mathrm{D}(n))})}  \geq X^i_{(\ceil*{n(q+\mathrm{D}(n))})}$ occurs we eliminate the arm $i$ and the arm $j $ remains in $\mc{A}$. The stopping condition and the inequalities in \eqref{eq:pairs_of_inequalities} guarantee that \begin{align}
    F_j^{-1}(q) \geq   F_i^{-1}(q).
\end{align} As a consequence the optimal arm $i^*$ is not eliminated and the Algorithm stops when $\mc{A}=\{i^*\}$. \end{proof}


The next result bounds the total number of pull at termination with high probability.

\begin{theorem}\label{thm:upper_bound}
Fix $\delta\in (0,1)$. There exists a constant $C>0$ such that the number of samples $\tau$ (and total number of pulls) of Algorithm \ref{alg:succ_el_quant} satisfies, with probability at least $1-\delta$, \begin{align}\label{eq:upper_bound}
    \tau \leq C \sum_{i\in \mc{A}^{-i^*}}\frac{\log\frac{K}{\delta}+\log(\frac{1}{\Delta_i })}{\Delta_i^2}.
\end{align}
\end{theorem} 

From the proof of Theorem \ref{thm:upper_bound}, it also follows that the number of pulls for each suboptimal arm $i\in\mc{A}^{-i^*}$ is at most $ \mc{O} \lp \log(K/\delta \Delta_i) /\Delta_i^2 \rp.$ The upper bound indicates that the number of pulls (with high probability) is proportional to the quantity $1/\Delta^2_i$ up to a logarithmic factor for each suboptimal arm $i$. In fact, experimental results on SEQ (Algorithm \ref{alg:succ_el_quant}) show that the explicit bound in \eqref{eq:upper_bound} matches the average number of pulls in the experiment. Next we provide the proof of Theorem \ref{thm:upper_bound}. For the simulation results we refer the reader to Section \ref{app:experiments}, Figure \ref{fig:simulations_matching_bounds}.

\begin{proof}[Proof of Theorem \ref{thm:upper_bound}] Under the event $\mc{E}$, we will find a bound on the smallest value of $n$ that satisfies the inequality $X^{i^*}_{(\floor*{n(q-\mathrm{D}(n))})}\geq X^i_{(\ceil*{n(q+\mathrm{D}(n))})}$ for all $i\neq i^*$. When $\mc{E}$ occurs, it is true that  \begin{align}\label{eq:ceil_ineq}
    X^{i^*}_{(\floor*{n(q-\mathrm{D}(n))})}\nonumber &\geq  X^{i^*}_{(\ceil*{n(q-\mathrm{D}(n))-1})}\\& \stackrel{(\text{A})}{\geq} F^{-1}_{i^*}(q-2\mathrm{D} (n)-1/n),  
\end{align} \begin{align}\nonumber
    F^{-1}_{i}(q+2\mathrm{D} (n)+1/n) &\stackrel{(\text{B})}{\geq}  X^i_{(\floor*{n(q+\mathrm{D}(n)+1/n)})}\\& \geq X^i_{(\ceil*{n(q+\mathrm{D}(n))})}\label{eq:floor_ineq}
\end{align} 
and (A), (B) come from the definition of $\mc{E}$.
From the definition of the suboptimality gap follows that $  F_{i^*}^{-1}(q-\Delta_{i^*})\geq F_{i}^{-1}(q+\Delta_i)$. The latter together with \eqref{eq:ceil_ineq} and \eqref{eq:floor_ineq} give that it is sufficient to find the smallest value of $n $ that satisfies the inequalities \begin{align}\label{eq:inequality_gap_n}
   F^{-1}_{i^*}(q-2\mathrm{D} (n)-1/n) \geq F_{i^*}^{-1} (q-\Delta_i),\\ F^{-1}_{i}(q+\Delta_i) \geq  F^{-1}_{i}(q+2\mathrm{D} (n)+1/n).\label{eq:inequality_gap_n2}
\end{align} We denote by $\tau_i$ the total number of pulls for a suboptimal arm $i$. The monotonicity of $F_{i^*}(\cdot)$, $F_{i}(\cdot)$ and \eqref{eq:inequality_gap_n}, \eqref{eq:inequality_gap_n2} give 
    \begin{align}
    \!\!\!\Delta_i \geq 2 \mathrm{D}(n)+\frac{1}{n}\implies \Delta_i \geq 2\sqrt{\frac{  \log(4Kn^2/\delta)}{2n}}+1/n,
    \end{align} 
and the values of $n$ that satisfy the inequality above are bounded by \begin{align}
   \tau_i = \mc{O} \lp \frac{\log\frac{K}{\delta \Delta_i }}{\Delta_i^2} \rp.
\end{align} To conclude, the total number of samples $\tau$ is $\sum_{i\in \mc{A}^{-i^*}} \tau_i$ with probability at least $1-\delta$.
\end{proof}

We next present a lower bound on the expected number of pulls.  Sz\"{o}r\'{e}nyi et al.~\cite{Szorenyi15mab} use the results of Mannor and Tsitsiklis~\cite{mannor2004sample} to obtain a bound that depends on $\max\{\varepsilon , \Delta_i\}$ (for some chosen $\varepsilon>0$). We use the approach suggested in the book of Lattimore and Szepesv\'{a}ri \cite{LSbandits_book:20} on a different class of distributions and get a bound that depends only on $\Delta_i$.



\begin{theorem}
\label{thm:lower_bound}
Fix $\delta \in (0, 1)$. There exists a quantile bandit with $K$-arms and gaps $\Delta_i\in (0, 1/4)$, $i\in [K]^{-i^*}$, such that
\begin{align}
	\inf_{\delta\mathrm{-PAC\ policy}\text{ }\pi}\E_{\pi}[\tau] \ge \sum_{j = 1, j\neq i^*}^{K} \frac{3(1-q)^2}{100 \Delta^2_j} \log\left( \frac{1}{4\delta} \right).
	\end{align}
\end{theorem}


From Theorem \ref{thm:lower_bound}, it follows that, up to logarithmic factors depending on $\delta$ (Theorem \ref{thm:lower_bound}), and also $K$, $\Delta_i,i\in\mc{A}^{-i^*}$ (Theorem \ref{thm:upper_bound}), Algorithm \ref{alg:succ_el_quant} is (almost) optimal relative to the expected number of pulls achieved, and its performance is \textit{necessarily} inversely proportional to the square of our suboptimality gap. 
More interestingly, our lower bound shows that as $\Delta_i \to 0$ the sample complexity goes to $\infty$ and indeed as Theorem \ref{lem:positive_gap} shows, $\Delta_i = 0$ implies that the best-quantile arm identification problem is impossible.

\begin{proof}[Proof of Theorem 4] We note that to prove a minimax lower bound we need only show a ``bad instance'' of the problem. It is convenient for the proof to use a mixed discrete/continuous distribution since the calculations are easier. We therefore define the following class of distributions:
	\begin{align}\label{eq:g_small}
	g^{w}(x) \triangleq w \delta(x) + (1 - w), \qquad x \in [0,1],
	\end{align}
i.e., a mixture of a mass (Dirac delta) at $0$ and a uniform distribution on $[0,1]$. Let $G^{w}$ be the cumulative distribution function of $g^{w}$. 
The KL-divergence between two such distributions is
	\begin{align}
	\KL{g^{w}}{g^{w'}} = w \log \frac{w}{w'} + (1-w) \log \frac{1 -  w}{1 - w'},
	\end{align}
which is the same as the divergence between two Bernoulli random variables.
The gap between $g^{w}$ and $g^{w + \gamma}$ for $q > w + \gamma$ and small $\gamma < \frac{1}{2} (q - w)$ is $\Theta(\gamma)$. To see this, let $\nu = g^w$ and $\nu' = g^{w + \gamma}$, so $\nu$ has the higher $q$-quantile. We can calculate the $(q-\eta)$-quantile of $\nu$ and the $(q+\eta)$-quantile of $\nu'$ as
    \begin{align}
    x &= \frac{ q - \eta - w }{ 1 - w } \\
    x'&= \frac{ q + \eta - (w + \gamma) }{ 1 - (w + \gamma) }.
    \end{align}
We need to find the inf over all $\eta$ such that $x' < x$. By taking the case of equality, we find
    \begin{align*}
    \frac{ q - \eta - w }{ 1 - w }
        &= \frac{ q + \eta - (w + \gamma) }{ 1 - (w + \gamma) }\iff \\
    (q - \eta - w) (1 - w - \gamma)
        &= (q + \eta - w - \gamma) ( 1 - w ) \iff \\
    \gamma (1 - q) &= \eta (2 - 2w - \gamma) \iff \\
    \eta &= \frac{1 - q}{2 - 2w - \gamma} \gamma  \label{eq: eta_gamma}\numberthis
    \end{align*}

We adapt a strategy for the mean-bandit problem appearing in ~\cite[Section 33.2]{LSbandits_book:20} to the quantile bandit setting. Let $\mc{E}$ denote a class of environments for the bandit problem and $\nu \in \mc{E}$ be a particular environment (i.e. setting of the arm distributions). Let $i^*(\nu)$ be the optimal arm\footnote{For the example in Theorem 4 there is a unique optimal arm.} 
which we will denote by $i^*$ when $\nu$ is clear from context. 	


Fix $\gamma < 1/6$. 
Recall that $G^w$ is the CDF of $g^w$ given by \eqref{eq:g_small}. Let $\nu^{(1)}$ be defined by the arm CDFs
	\begin{align}
	\nu_{i}^{(1)} = \begin{cases}
		G^{1/3 - \gamma} & i = 1\\
		G^{1/3} & i \ne 1.
		\end{cases}
	\label{eq:lb:instance1}
	\end{align}
The gap between $\nu_{1}^{(1)}$ and $\nu_{i}^{(1)}$ is (setting $w = 1/3 - \gamma$ and using the fact that $\gamma < 1/6$):
    \begin{align}
    \Delta_i = \frac{1 - q}{2 - 2w - \gamma} \gamma
    = \frac{1 - q}{4/3 + \gamma} \gamma
    \ge \frac{3}{5} (1 - q) \gamma.
    \label{eq:lb:gapgamma}
    \end{align}
For each $j$ define $\nu^{(j)}$
	\begin{align}
		\nu_{i}^{(j)} = \begin{cases}
		G^{1/3 - \gamma} & i = 1 \\
		G^{1/3 - 2 \gamma} & i = j \\
		G^{1/3} & i \ne 1,j
		\end{cases}.
	\label{eq:lb:instancej}
	\end{align}

Let $\pi$ be a $\delta$-PAC policy. Then we have $\mathbb{P}_{\nu^{(1)}}( \hat{k} \ne 1 ) \le \delta$ and $\mathbb{P}_{\nu^{(j)}}( \hat{k} \ne j ) \le \delta$. Since $\nu^{(1)}$ and $\nu^{(j)}$ differ in only a single arm distribution, we have~\cite[Lemma 15.1]{LSbandits_book:20}
	\begin{align*}
	\!\!\KL{\mathbb{P}_{\nu^{(1)}}}{ \mathbb{P}_{\nu^{(j)}} }
	&= \sum_{i=1}^{K} \E_{\nu^{(1)}}[T_i(n)] \KL{ P_{\nu^{(1)}_i} }{P_{\nu^{(j)}_i} }\\
	&= \E_{\nu^{(1)}}[T_j(\tau)] \KL{ G^{1/3} }{ G^{1/3 - 2\gamma} },\numberthis
	\end{align*}
and
	\begin{align*}
	&\KL{ G^{1/3} }{ G^{1/3 - 2\gamma } }
	\\&= \frac{1}{3} \log \frac{ 1/3 }{1/3 - 2\gamma} + \frac{2}{3} \log \frac{ 2/3 }{2/3 + 2\gamma } \\
	&= \frac{1}{3} \log \frac{1}{1 - 6 \gamma} + \frac{2}{3} \log \frac{1}{ 1 + 3 \gamma } \\
	&\le \frac{1}{3} \left( 6\gamma + 18 \gamma^2 + 54 \gamma^3 \right)
		-  \frac{2}{3} \left( 3 \gamma - \frac{9}{2} \gamma^2 \right) \\
	&= 9 \gamma^2 + 18 \gamma^3,\numberthis
	\end{align*}
where we used the inequalities 
$\log \frac{1}{1 - x} \le x + \frac{x^2}{2} + \frac{x^{3}}{3}$ 
and 
$- \log (1 + x) \le -x + \frac{x^2}{2} - \frac{x^3}{4} 
\le -x + \frac{x^2}{2}$ 
for $x \in [0, 0.42]$. So for $\gamma < \frac{1}{6}$, 
	\begin{align}
	\KL{\mathbb{P}_{\nu^{(1)}}}{ \mathbb{P}_{\nu^{(j)}} } \le 12 \gamma^2 \E_{\nu^{(1)}}[T_j(\tau)].
	\end{align}
	
\noindent Now define the events
	\begin{align}
	A &=  \{\tau < \infty\} \cap \{\hat{k} \ne j \}  \\
	A^c &= \{ \tau = \infty \} \cup \{ \hat{k} = j \} .
	\end{align}
Then since $A^c \subseteq \{ \hat{k} \ne 1 \}$ and $\pi$ is $\delta$-PAC policy we have $\mathbb{P}_{\nu^{(1)}}( A^c ) + \mathbb{P}_{\nu^{(j)}}( A ) \le 2 \delta$. Now, by the Bretagnolle-Huber Inequality~\cite[Theorem 14.2]{LSbandits_book:20},
	\begin{align}\nonumber
	2 \delta &\ge \frac{1}{2} \exp\left( - \KL{\mathbb{P}_{\nu^{(1)}}}{ \mathbb{P}_{\nu^{(j)}} } \right) \\
	&\ge \frac{1}{2} \exp\left( - 12 \gamma^2 \E_{\nu^{(1)}}[T_j(\tau)] \right).
	\end{align}
By rearranging and using \eqref{eq:lb:gapgamma} to get an upper bound on $\gamma$ in terms of the gap $\Delta_i$
	\begin{align}\nonumber
	\E_{\nu^{(1)}}[T_j(\tau)]&\ge \frac{1}{15 \gamma^2} \log\left( \frac{1}{4\delta} \right)
	 \\&\ge \frac{3}{100 \Delta_i^2} (1 - q)^2 \log\left( \frac{1}{4\delta} \right).
	\end{align}
Repeating the argument for each $j \in \{2,3,\ldots K\}$ we get
    \begin{align}\nonumber
	\E_{\nu^{(1)}}[\tau] &= \sum_{j = 1, j\neq i^*}^{K} \E[T_j(\tau)] \\&\ge  \sum_{j = 1, j\neq i^*}^{K} \frac{3(1-q)^2}{100 \Delta^2_j} \log\left( \frac{1}{4\delta} \right),\label{eq:proof_converse_last_inequality}
	\end{align}and \eqref{eq:proof_converse_last_inequality} gives the bound of the theorem. 	\end{proof}

\begin{remark}
We leave proving an instance-based lower bound as future work. We believe this will be quite challenging, since knowing only the value of the gap at quantile $q$ gives only local information about the CDF of the distribution.
\end{remark}

\section{A Private Algorithm for\\Best-Quantile-Arm Identification}
\label{sec:dp-best-quantile}

We now turn to the privacy-preserving version of our best-arm identification algorithm. Bandit problems using private data arise naturally in medical and financial contexts, and privacy for online/sequential learning problems remains an active area of research. We provide results in this section on differentially private bandit learning. In differential privacy, the \emph{privacy guarantees} should hold for any value of the input data. However, \emph{utility guarantees} are made under the assumption that the rewards come from a stochastic process. For bandit learning, this means that our privacy guarantees will hold any realization of the arms' rewards and our bound on the number of pulls will depend on the distribution of the arms' rewards. The monograph of Dwork and Roth~\cite{DworkRoth} provides an excellent introduction to the fundamentals of differential privacy.

To derive our privacy results, we think of the rewards from each of the arms at each time $t$ as coming from different individuals. This means that to protect an individual we are interested in \emph{event-level privacy}, defined for private algorithms operating on streams~\cite{DworkNPR10}. Let $\mbf{X} = [\mbf{X}_1, \mbf{X}_2, \ldots, \mbf{X}_K]^{\top}$ be a collection of $K$ (infinite) sequences of rewards and let the $n$-th reward of arm $i$ be denoted by $(X^i_n)$.

\begin{definition}
\label{def:dp_BAI_nbr}
Two sequences of rewards $\mbf{X} = (X^i_n)_{i,n}$ and $\mbf{X}' = (X'^i_n)_{i,n}$ are called \emph{neighboring} (denoted by $\mbf{X} \sim \mbf{X}'$) if there exists only a single pair $(i,n)$ for which $X^i_n \ne X'^{i}_n$.
\end{definition}

We note that this definition of neighboring for differentially private bandit problems is standard~\cite{ShariffS18,SajedS19} and we use the model of differential privacy under continual observation to handle the streaming setting. 

\begin{definition}
\label{def:dp_BAI_alg}
A randomized algorithm $A$ is said to be \emph{$\eps$-differentially private ($\eps$-DP) under continual observation} if for any two neighboring rewards $\mbf{X}$ and $\mbf{X}'$ and for any set $\mc{S}$ of outputs of the algorithm, we have $\Pr[A(\mbf{X})\in \mc{S}]\leq e^\eps \Pr[A(\mbf{X}')\in \mc{S}]$.
\end{definition}

We can view $\mbf{X}$ as a sequence of column vectors of rewards indexed by time. In the continual observation setting~\cite{DworkNPR10} the algorithm accesses these vectors sequentially (one entry per column based on the arm chosen by the algorithm) and hence the output $A(\mbf{X})$ is also revealed one pull at a time. More specifically, a private bandit algorithm reveals which arm it chooses to pull at each time, so the overall output of an algorithm for best-arm identification is the sequence of pulled arm indices as well as the identified best arm. 
The advantage of the streaming definition~\cite{DworkNPR10} is that the algorithm's privacy guarantees can be made for pairs of neighboring streams $\mbf{X}$ and $\mbf{X}'$ without the algorithm having to know termination time in advance. That is, when the algorithm terminates and the output is fully revealed, it guarantees the same probability bound to every pair of neighboring streams $\mbf{X}$ and $\mbf{X}'$.  

Differentially private algorithms are randomized in order to guarantee that the outputs do not depend too strongly on individual data points in the input. This randomization is internal to the algorithm: the privacy guarantee has to hold for any pair of neighboring input streams $\mbf{X}$ and $\mbf{X}'$. The guarantee implies that an adversary, when viewing the output of the algorithm, will not be able to infer whether the input data was $\mbf{X}$ or $\mbf{X}'$, even if all of the common entries of $\mbf{X}$ and $\mbf{X}'$ are revealed. This is a strong guarantee which has led to a large body of work on differentially private learning.

Since differential privacy is a property of algorithms, a common approach to privately approximating a statistic (sometimes called a query) $f(\mbf{X})$ is to compute $f(\mbf{X})$ and add noise. A fundamental quantity of interest is the \emph{global sensitivity} $G(f) = \max_{\mbf{X} \sim \mbf{X}'} |f(\mbf{X}) - f(\mbf{X}')|$, which measures how much $f(\mbf{X})$ can change between neighboring inputs. If $G(f) \le 1$ then the algorithm which outputs $f(\mbf{X}) + Z$ is $\epsilon$-differentially private if $Z$ has a Laplace distribution $\mathsf{Lap}(1/\epsilon)$ with density $\frac{\epsilon}{2} e^{-\epsilon |z|}$. Unfortunately, the quantile functions (or quantile queries) have a very high global sensitivity. Taking the median query as an example, changing a single sample (in the worst case) can change the median of the set $\{0, 0, 0, M, M\}$ from $1$ to $M$, meaning $G(f) = M$, which is full range of the data. As we discuss below, this makes privately computing quantiles challenging.


Differentially private algorithms also enjoy certain composition properties~\cite{DworkRoth} (which we will use in the analysis of our proposed bandit scheme) that make them attractive for use in privacy settings. The first is \textit{basic composition:} if $A_1$ and $A_2$ are $\eps_1$- and $\eps_2$-DP algorithms resp., then for any $\mbf{X}$ releasing the pair $\{ A_1(\mbf{X}), A_2(\mbf{X})\}$ is $(\eps_1+\eps_2)$-DP (provided both algorithms use independent randomization). \textit{Parallel composition} implies that if $A$ is an $\eps$-DP algorithm, then for any input $S$ and any column-wise partition of $\mbf{X}$ into $\mbf{X}_1, \mbf{X}_2$, outputting $\{ A(\mbf{X}_1), A(\mbf{X}_2)\}$ is $\eps$-DP (again, provided both algorithms are run using independent randomization).





\subsection{Differential privacy and quantiles \label{subsec:dp_preliminaries}}

\newcommand{\cut}[1]{}
Releasing a differentially private \emph{estimate} of the $q$-quantile of a given distribution is considered to be a hard task. Tight bounds for $\eps$-differential privacy were given by Beimel, Nissim, and Stemmer~\cite{BeimelNS13} and Feldman and Xiao~\cite{FeldmanX14}, with the accuracy dependent on the cardinality of the distribution's support. This makes the problem infeasible for continuous distributions such as those supported on $[0,1]$.  The algorithm we propose gets around this by never publishing an approximation for $q$-quantile; instead we output an arm $\hat{k}$ that should have a higher $q$-quantile than any other arm $i$. 
To do this, we eliminate an arm by (privately) estimating the number of \emph{pairs of draws} attesting for an arm's suboptimal $q$-quantile. This function/query is a counting query, whose global sensitivity is always $1$ regardless of the size of the support of the reward distribution of arm $i$. This reformulation is what allows us to obtain a sample complexity bound that is independent of the support size of any arm's distribution and hence works for continuous distributions, even with unbounded support.

\textit{On the difficulties with a private UCB quantile algorithm.} Differentially private UCB algorithms for the mean using tree-based algorithms~\cite{ChanSS10,DworkNPR10} do not extend straightforwardly to the quantile case, but a carefully designed counting query\footnote{Count the number of examples required to make the quantile-UCB of this arm the max.} makes using tree-based algorithms feasible to our problem. However, our proposed Algorithm \ref{alg:DP_SE_Quantiles} is superior to this approach in two respects, both related to the horizon $T$. First, (as observed by Sajed and Sheffet~\cite{SajedS19}) the tree based algorithm's utility bound has a ${\rm poly}(\log(T))$ dependence whereas our algorithm's is only $\log\log(T)$.\footnote{Both utility guarantees also have a $\log(1/\delta)$-factor.} Secondly, the tree-based algorithms require knowing $T$ in advance; this is nontrivial because doubling tricks require either rebudgeting $\epsilon$ (incurring increased sample complexity) or discarding all samples when the next epoch begins, which incurs $\tilde O(\Delta_i^{-2})$ pulls per suboptimal arm in \emph{every} epoch because the UCB algorithm never eliminates any arms. Our gap definition and algorithm avoids having any such prior knowledge of $T$ or the value of the gap.

\noindent\textit{Notation.} Throughout this section we deal with pure $\eps$-DP and use $\delta$ to represent the failure probability of our algorithm. The reader is advised to not be confused with the notion of $(\eps,\delta)$-DP.\footnote{We could have used the notion of approximate $(\eps,\delta)$-DP to reduce our total privacy loss by a factor of $\sqrt K$ by relaying on the advanced composition theorem~\cite{Dwork10:boosting,KairouzOV:15composition}. As a matter of style, we opted for pure-DP.} 

\subsection{Differentially Private Successive Elimination for the Highest Quantile Arm} \label{subsec:DP_BAI_for_quantiles}

The differentially private algorithm is shown in Algorithm \ref{alg:DP_SE_Quantiles}. Much like the algorithm in Sajed and Sheffet~\cite{SajedS19}, our algorithm is also epoch based. In epoch $e$ our goal is to eliminate all arms $i$ with gap (from \eqref{eq:gap_distributions}) $\Delta_i \geq \epochgap = 2^{-e}$. As we argue, the number of arm pulls in each epoch from each existing arm is $n_e\geq \epochgap^{-2}$. The key point is that due to the geometric nature of $\epochgap$ it follows that each $n_e$ is proportional to the sum of pulls thus far $\sum_{1\leq e' < e} n_{e'}$, and so we may as well split the stream into different chunks, starting each epoch anew (discarding all examples drawn in all previous epochs). Because we eliminate arms, this still doesn't cost us a lot in the number of overall pulls, yet allows us to avoid splitting the privacy budget $\epsilon$ due to parallel composition.

\begin{algorithm}[hbt]
\renewcommand{\algorithmicrequire}{\textbf{Input:}}
	\caption{\label{alg:DP_SE_Quantiles}
	Differentially Private Successive Elimination for Quantiles (DP-SEQ)}
	\begin{algorithmic}[1]
	\Require Number of arms $K$, quantile level $q\in (0,1)$, privacy parameter $\eps>0$, failure probability $\delta\in (0,1/2)$.
	\State Initialize $\mc{A}\leftarrow\{1,\ldots,K \}$
	\State  $e_*=\min_{e\in\mbb{N}}\{ e: e\geq \max\{-\log_2 (1-q) , -\log_2 (q) \}-1\}$
	\State epoch $e\gets e^*-1$
	\While{$|\mc{A}|>1$}
	\State Increment $e\gets e+1$
	\State  $\epochgap \gets 2^{-e}$, $\gamma \gets \frac{\epochgap}4$
\State $n_e \gets \max\left\{\frac{16}{\epochgap^2}, \frac{64(|\mc{A}|-1)}{\epochgap\cdot \eps}\right\}\cdot \log(\frac{6|\mc{A}|e^2}\delta)$
	\For{$a \in \mc{A} $}
	    \State Pull arm $a$ for $n_e$ times to obtain $X^a_{1}, X^a_{2},..., X^a_{n_e}$.
	    \State Order samples into $X^a_{(1)}, X^a_{(2)},..., X^a_{(n_e)}$
	\EndFor
    \State $i \gets \lfloor n_e (q - 2\gamma)\rfloor$ and $j \gets \lceil n_e (q + 2\gamma)\rceil$.
	\For{ $(a,b) \in \mc{A} \times \mc{A}$ such that $a\neq b$}
	        \State $Z_{a,b} \sim \mathsf{Lap}(\frac{2(|\mc{A}|-1)}\eps)$   
	        \Comment{$\mathsf{Lap}(\beta)=\frac{1}{2\beta} e^{-|z|/\beta}$}
	        \vspace{+0.1cm}
        \State $\ell^* \gets \max \{ 0 \le \ell \leq \min\{i, n_e-j\} $ such that $ \hphantom1   \kern80pt  X^b_{(j+\ell')} \leq X^a_{(i-\ell')} \ \forall \ell'\leq \ell \}$
	\State \textbf{if} { $\max\{0,\ell^*\} + Z_{a,b} \geq \frac{4(|\mc{A}|-1)}\eps  \log(6|\mc{A}|^2e^2/\delta)$ }  {\hphantom1 \kern21pt \textbf{then} $\mc{A}\leftarrow \mc{A}\setminus\{b\}$} 
	\EndFor
\EndWhile
\end{algorithmic}
\end{algorithm}

We still need a way to privately eliminate arms at the end of each epoch. 
In the case of the means, Sajed and Sheffet~\cite{SajedS19} eliminate arms by computing $\eps$-DP approximations of the means and comparing those, leveraging the post-processing invariance of DP. Unfortunately, we cannot find $\eps$-DP approximations for $q$-quantiles that do not depend on the cardinality of the support. Instead, we resort to the more naive approach of pairwise comparisons between all $K(K-1)/2 = \Theta(K^2)$ pairs of arms. This requires partitioning the $\eps$ of our privacy budget into $ \epsilon/2(K-1)$ as each arm participates in at most $2(K-1)$ many comparisons. However, using pairwise comparisons we are able to convert the higher-quantile question into a counting query: how many consecutive examples satisfy that $X^a_{(\mathrm{LCB}-i)}\geq X^b_{(\mathrm{UCB}+i)}$? Here $\mathrm{LCB}$ is the index of the lower confidence bound and $\mathrm{UCB}$ of the upper confidence bound.
We prove that under event-level privacy, this query has sensitivity of at most $1$, allowing us to eliminate the suboptimal arm $b$ via the standard Laplace mechanism.

Our first result for differential privacy is a guarantee for Algorithm \ref{alg:DP_SE_Quantiles}. 

\begin{theorem}
\label{thm:Alg_is_DP}
Algorithm~\ref{alg:DP_SE_Quantiles} is $\eps$-differentially private under continual observation.
\end{theorem} 


\begin{proof}[Proof of Theorem \ref{thm:Alg_is_DP}]
Let $A$ denote the algorithm. Fix two neighboring input streams $\mbf{X}$ and $\mbf{X}'$ and suppose that they differ in the entry $X_{n}^a$ corresponding to time $n$ and arm $a$. Let $e$ denote the epoch in which the time index $n$ falls. Since the rewards are identical up to epoch $e$, the distribution of outputs of $A(\mbf{X})$ and $A(\mbf{X}')$ are identical up to epoch $e$. In comparing the probabilities under inputs $\mbf{X}$ and $\mbf{X}'$ we may therefore condition on the set $\mc{A}$ of available arms at the beginning of epoch.

Now let us consider epoch $e$ for the stream $\mbf{X}$ and define $i$ and $j$ as in Algorithm \ref{alg:DP_SE_Quantiles}. After each epoch we compare each pair of arms, so consider a pair of arms $a$ and $b$. If neither $X_{(j)}^a \le X_{(j)}^b$ nor $X_{(j)}^b \le X_{(i)}^a$, then set $\ell^* = 0$. Otherwise without loss of generality assume $X_{j}^b \le X_{i}^a$ and set $\ell^*$ to be the smallest element of the set $\max \{ 0 \le \ell \leq \min\{i, n_e-j\} \colon \forall \ell'\leq \ell \  X^b_{(j+\ell')} \leq X^a_{(i-\ell')} \}$. We claim this function has global sensitivity $1$.

We must evaluate how much $\ell^*$ can change by changing one sample from $\mbf{X}$ to a neighboring $\mbf{X}'$ with index $\ell'^*$. Without loss of generality, assume the shifted reward is in arm $a$ so the rewards on arm $b$ are identical. 
We have the following sequence of inequalities for $\ell^*$
\begin{align}\nonumber
   \!\!\!\!\!\!  X^b_{(j)} \!\leq\! X^b_{(j+1)} \!\leq \!\ldots\! &\leq X^b_{(j+\ell^*)} \!\leq\! X^a_{(i-\ell^*)}\\ &\leq X^a_{(i-\ell^*+1)} \!\leq \!\ldots\! \leq  X^a_{(i)} \!\leq\!  X^a_{(i+1)}.\!\!\!\label{eq:order_sequence}
\end{align}
Since $\ell^*$ is maximal, we know that $X^b_{(j+\ell^*+1)} > X^a_{(i-\ell^*-1)}$, giving us the chain of inequalities
    \begin{align}
    X^b_{(j+\ell^*+2)} \geq X^b_{(j+\ell^*+1)}> X^a_{(i-\ell^*-1)}\geq X^a_{(i-\ell^*-2)}.
    \label{eq:rev_sequence}
    \end{align}
Now consider the rewards in $\mbf{X}'$ and the set of indices $\mc{T} = [i - \ell^* -1, i]$. The sets $\{X'^a_{(t)} : t \in \mc{T}\}$ and $\{X^a_{(t)} : t \in \mc{T}\}$ differ in at most a single element.
If they do not differ then they satisfy \eqref{eq:order_sequence} and \eqref{eq:rev_sequence} so $\ell'^* = \ell^*$. If they do differ then the two sequences 
    \begin{align*}
    \!\!\!\!\!\! X^a_{(i - \ell^* - 1)} 
        \le X^a_{(i - \ell^*)}
        \le X^a_{(i - \ell^* + 1)}
        \le\! \ldots\!
        \le X^a_{(i)}
        \le X^a_{(i+1)} \\
    \!\!\!\!\!\! X'^a_{(i - \ell^* - 1)} 
        \le X'^a_{(i - \ell^*)}
        \le X'^a_{(i - \ell^* + 1)}
        \le\! \ldots\!
        \le X'^a_{(i)}
        \le X'^a_{(i+1)}
    \end{align*}
are shifted by at most one position. Suppose that $X'^b_{(j+\ell^*)} > X'^a_{(i-\ell^*)}$. Since $X'^b_{(j+\ell^*)} = X^b_{(j+\ell^*)} \le X^a_{(i-\ell^*)}$, this implies $X^a_{(i-\ell^*)} > X'^a_{(i-\ell^*)}$. Since the sequences are shifted by at most $1$, we have $X^a_{(i-\ell^*)} \le X'^a_{(i-\ell^*+1)}$. Then we have $X'^b_{(j + \ell^* - 1)} \le X^a_{(i-\ell^*)} \le X'^a_{(i-\ell^*+1)}$ which implies $\ell'^* \ge \ell^* - 1$.

Now suppose $X'^b_{(j+\ell^*+1)} \le X'^a_{(i-\ell^*-1)}$, which implies that $X^a_{(i-\ell^*-1)} < X'^a_{(i-\ell^*-1)}$. Since the sequences are shifted by at most $1$, $X'^a_{(i-\ell^*-2)} < X^a_{(i-\ell^*-1)}$ and we have $X'^b_{(j+ \ell^* + 2)} > X^a_{(i-\ell^*-1)} > X'^a_{(i-\ell^*-2)}$, showing that $\ell'^* \le \ell^* + 1$.





We have shown that the function $\ell^*(\mbf{X})$ has sensitivity $1$, so we can apply the Laplace noise mechanism. Define $\eps'\triangleq \eps/2(|\mc{A}|-1)$. It follows then that the differentially private approximation $\ell^* + \mathsf{Lap}(1/\eps')$ preserves $\eps'$-DP. Since arm $a$ participates in at most $2(|\mc{A}|-1)$ many such queries in epoch $e$, we have by direct composition that our algorithm is $\eps$-DP.
\end{proof}

We continue by providing a high probability guarantee on the first epoch for which the private SEQ (Algorithm \ref{alg:DP_SE_Quantiles}) terminates.

\begin{theorem}
\label{thm:utility_DP_quantile_SE}
For Algorithm~\ref{alg:DP_SE_Quantiles}, the following events occur with probability at least $1-\delta$: (a) it keeps at least one optimal arm in $\mc{A}$ and (b) it removes each suboptimal arm $a$ by epoch $e=\ceil{\log_2(1/\Delta_a)}$.
\end{theorem}

\begin{proof}[Proof of Theorem \ref{thm:utility_DP_quantile_SE}]
Fix an epoch $e$, constant $\epochgap = 2^{-e}$ and let $\delta > 0$.  
We denote the following ``bad'' events at the end of the epoch,
\begin{align*}
    E_1 &\triangleq\big\{ \exists a \in\mc{A}:
        X^a_{(i)}>F_a^{-1}\lp i/n_e + \epochgap / 4 \rp 
        \\&\qquad\qquad\qquad\textrm{ or }  
        X^a_{(\lfloor i-n_e\frac{\epochgap}4\rfloor )}< F_a^{-1}\lp i/n_e - \epochgap / 4\rp \big\},
    \cr E_2 &\triangleq\big\{ \exists a \in\mc{A}:
        X^a_{(j)}< F_a^{-1}\lp j/n_e-\epochgap / 4\rp 
        \\&\qquad\qquad\qquad\textrm{ or } 
        X^a_{(\lceil j+n_e\frac{\epochgap}4\rceil )} > F_a^{-1}\lp j/n_e + \epochgap / 4\rp \big\},
    \cr E_3 &\triangleq\big\{ \exists (a,b)\in\mc{A}^2: |Z_{a,b}| >   \frac{2(|\mc{A}|-1)}\eps\log(6|\mc{A}|^2e^2/\delta)\big\}.
\end{align*}
We have $n_e \ge \frac{16}{\epochgap^2} \log\left( \frac{6 |\mc{A}| e^2}{\delta} \right) \ge \frac{8}{\epochgap^2} \log\left( \frac{12 |\mc{A}| e^2}{\delta} \right)$, so we can apply Lemma \ref{thm:GCB} (Appendix \ref{app:properties}) with $\zeta = \epochgap/4$ to show that for a given arm $a$ and any specific index $k$, it holds that
    \begin{align}
    \Pr\left[ X^a_{(k)}<F_a^{-1} \lp k/n_e -\frac{\epochgap}{4} \rp \right]
    &\leq \frac{\delta}{12e^2|\mc{A}|}, \\
    \Pr\left [X^a_{(k)} > F_a^{-1}\lp k/n_e+\frac{\epochgap}{4} \rp \right]
    &\leq \frac{\delta}{12e^2|\mc{A}|}.
    \end{align}
Applying the union bound over the $|\mc{A}|$ choices for an arm and the two particular indices $k=i$ and $k = \lfloor i - n_e \frac{\epochgap}4\rfloor$, we have that $\Pr[E_1]\leq \delta/6e^2$. Similarly, the same line of reasoning gives that $\Pr[E_2]\leq \delta/6e^2$. Lastly, due to the properties of the Laplace distribution (or the exponential distribution which dictates the magnitude of $|Z_{a,b}|$ we have that $\Pr[E_3]\leq |\mc{A}|^2 \delta/6e^2|\mc{A}|^2= \delta/6e^2$. We apply the union bound again (twice) to infer that $\Pr[E_1 \cup E_2 \cup E_3]\leq \delta/2e^2$, and thus, the probability that
    \begin{align}
    \Pr\left[\exists e: E_1, E_2 \textrm{ or } E_3\textrm{ occur}\right]
    \leq \sum_{e\geq 0}\delta/2e^2 \leq \delta.
    \end{align}

We continue under the assumption that in all epochs all three bad events never occur. Also by our choice of $n_e$ it is true that $8(|\mc{A}|-1)\eps^{-1} \log(6|\mc{A}|^2e^2/\delta) \leq 16(|\mc{A}|-1)\eps^{-1} \log(6|\mc{A}|e^2/\delta) \leq n_e \epochgap /4$. It is now fairly straightforward to argue that when comparing a suboptimal arm $a$ and an optimal arm $b$ we never remove $b$: this follows from the fact that in this case we have 
    \begin{align*}
    X^b_{(j)} &\geq F_b^{-1}\lp \frac j {n_e} - \frac{\epochgap} 4 \rp \\&\geq F_b^{-1}(q)> F_a^{-1}(q) \geq F_a^{-1}\lp\frac{i}{n_e} + \frac{\epochgap} 4\rp \geq X^a_{(i)}
    \end{align*}
and so for such a pair $\ell^* = 0$, making $\ell^*+Z_{a,b} \leq 2(|\mc{A}|-1)\log(6|\mc{A}|^2e^2/\delta)/\eps$ under the complement of $E_3$. Thus, we can only eliminate an optimal arm when comparing it to another optimal arm, and so $\mc{A}$ must always contain at least one optimal arm. Secondly, when comparing an optimal arm $a$ to a suboptimal arm $b$ where the optimality gap is at least $2^{-e}$ we have that at epoch $e$ it holds that for $\ell = 6(|\mc{A}|-1) \log(6|\mc{A}|^2e^2/\delta)\eps$ we have
\begin{align} \nonumber
    X^b_{(j)}\leq X^b_{(j+1)} \leq ... &\leq  X^b_{(j+\ell)} \\&\leq X^b_{(\lceil j+n_e \frac{\epochgap}4)\rceil} \leq F_b^{-1}(q+\epochgap)\label{eq:long1}\end{align} and
\begin{align}    \nonumber
    F_b^{-1}(q+\epochgap) \leq F_a^{-1}(q-\epochgap) &\leq X^a_{(\lfloor i-n_e\frac{\Delta}4\rfloor)} \\&\leq X^a_{(i-\ell)} \leq... \leq X^a_{(i)}.\label{eq:long2} \end{align}
It follows that for such a pair $\ell^* \geq 6(|\mc{A}|-1)\log(6|\mc{A}|^2e^2/\delta)/\eps$, under $E_3$ we have that $\ell^*+Z_{a,b} \geq 6(|\mc{A}|-1)\log(4|\mc{A}|^2e^2/\delta)/\eps$ so we eliminate arm $b$. The latter, \eqref{eq:long1} and \eqref{eq:long2} complete the proof.
\end{proof}

Lastly, we characterize the sample complexity of DP-SEQ (Algorithm \ref{alg:DP_SE_Quantiles}), the number of pulls for each suboptimal arm and the total number of pulls at termination.
\begin{theorem}
\label{thm:num_pulls_DP_alg}
With probability at least $1 - \delta$,  Algorithm~\ref{alg:DP_SE_Quantiles}, pulls each suboptimal arm $a$ at most 
\begin{align}
    \mc{O}\lp \lp \frac 1{\Delta_a^2}+\frac{K}{\eps\Delta_a} \rp  {\log\lp\frac{K}{\delta}\log\lp \frac 1 {\Delta_a}\rp\rp} \rp
\end{align}many times. 

\end{theorem}

By taking $\epsilon\to\infty$, Theorem \ref{thm:num_pulls_DP_alg} provides the utility of the standard (non-private) epoch-based successive elimination variant of Algorithm \ref{alg:succ_el_quant}. 
Indeed, by introducing epochs the concentration bound in \eqref{eq:concentration_bound_main} becomes \begin{align*}
   \P\big( F_i^{-1} \lp q\rp\in \big[X^i_{(\floor*{n_e(q-\mathrm{D}_e)})} , X^i_{(\ceil*{n_e(q+\mathrm{D}_e)})}  \big] \big) > 1-\frac{\delta}{2Ke^2},
\end{align*} $e$ denotes the epoch, $\mathrm{D}_e =2^{-e}$ and $n_e= \mathrm{D}^{-2}_e \log(4Ke^2/\delta)/2$. This yields a bound on the total number of pulls for the epoch-based algorithm of the order of 
\begin{align}
  \mc{O}  \Bigg( \sum_{i \in \mc{A}^{-i^*}}\frac 1{\Delta_i^2}  {\log\lp \frac{K}{\delta}\log\lp\frac 1 {\Delta_i}\rp\rp} \Bigg),
\end{align} 
matching the ($\varepsilon$-optimal) bounds of~\cite{HowardR19}.
As a consequence of the epoch-based approach, the dependence $\log(K/(\delta\Delta_i))$ in \eqref{eq:upper_bound} becomes $\log(K\delta^{-1} \log(\Delta^{-1}_i))$ for $i\in\mc{A}\setminus\{i^*\}$. However, this comes at the expense of much larger constants. We continue by presenting the proof of Theorem \ref{thm:num_pulls_DP_alg}.

\begin{proof}[Proof of Theorem \ref{thm:num_pulls_DP_alg}]
Fix any suboptimal arm $a$. 
Denote $e^*$ as the first integer $e$ for which $2^{-e}\leq \Delta_a$. Thus $\Delta_{e^*} = 2^{-e^*}\leq \Delta_a < 2\Delta_{e^*}$ making $2^{e^*}\leq 2/\Delta_a$. 
According to Theorem~\ref{thm:utility_DP_quantile_SE} we have that with probability at least $1-\delta$ by epoch $e^*$ arm $a$ is eliminated. Since in any epoch we have that $|\mc{A}|\leq K$, we have that the total number of pulls of arm $a$ is
\begin{align*}
    &\sum_{0\leq e\leq e^*}\!\!\! n_e 
    \cr&\leq \sum_{0\leq e\leq e^*} \left(\frac{16}{\epochgap^2}+ \frac{64(K-1)}{\epochgap\cdot \eps}\right)\cdot \log\lp \frac{6Ke^2}\delta\rp
    \cr &\! \leq 
    16\log\lp\frac{6K(e^*)^2}\delta\rp \!\!\sum_{0\leq e\leq e^*}\!\!\! 2^{2e} + \frac{64K\log\lp\frac{6K(e^*)^2}\delta\rp}{\eps} \!\!\!\sum_{0\leq e\leq e^*}\!\!\! 2^{e}
    \cr &\! \leq  
    32\log\lp\frac{6K(e^*)^2}\delta\rp 2^{2e^*} +\frac{128K\log\lp\frac{6K(e^*)^2}\delta\rp}{\eps}2^{e^*}
    \cr &\!\leq 
    \frac {128\log\lp\frac{6K(e^*)^2}\delta\rp}  {\Delta_a^2} +\frac{256K\log\lp\frac{6K(e^*)^2}\delta\rp}{\eps\Delta_a} 
    \cr 
    &\!\leq \left(\frac 1 {\Delta_a^2}+\frac{K}{\eps\Delta_a}\right)\cdot 512 \log\lp \frac{6K}{\delta} \cdot \log\lp\frac 1 {\Delta_a}\rp\rp.
\end{align*} To conclude, the total number of samples (and pulls) is \begin{align}
  \mc{O}  \lp \sum_{a \in \mc{A}^{-i^*}}\lp \frac 1{\Delta_a^2}+\frac{K}{\eps\Delta_a} \rp \cdot {\log\lp \frac{K}{\delta}\log\lp\frac 1 {\Delta_a}\rp\rp} \rp
\end{align}with probability at least $1-\delta$. 
\end{proof}

\section{Numerical Illustrations and Further Discussion of the Results \label{app:experiments}}

In this section, we provide indicative numerical simulations (along with relevant discussion) exploring and confirming various properties related to the proposed elimination algorithms (private and non-private), as well as the proposed definition of the associated suboptimality gap.

\begin{figure}[t!]
    \centering
    \vspace{-0.2cm}
    \includegraphics[width=0.45\textwidth, valign=t]{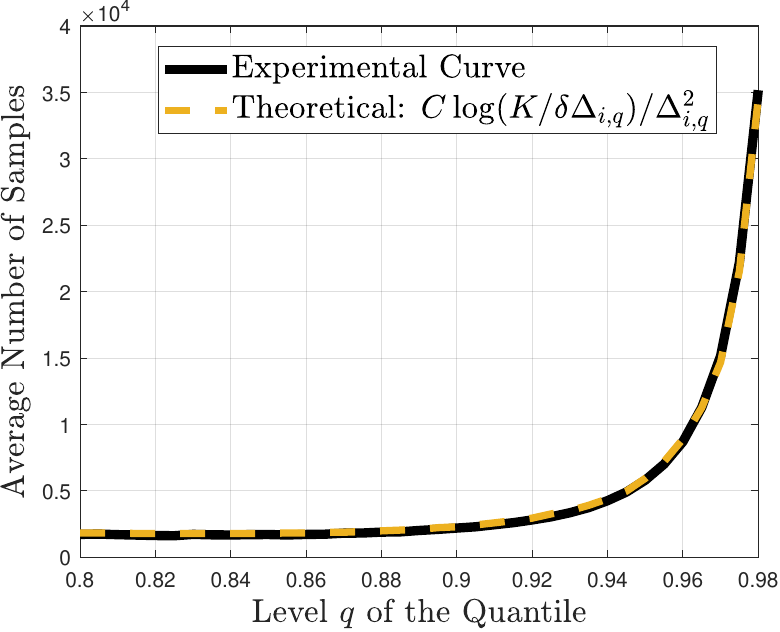}
    \includegraphics[width=0.45\textwidth, valign=t]{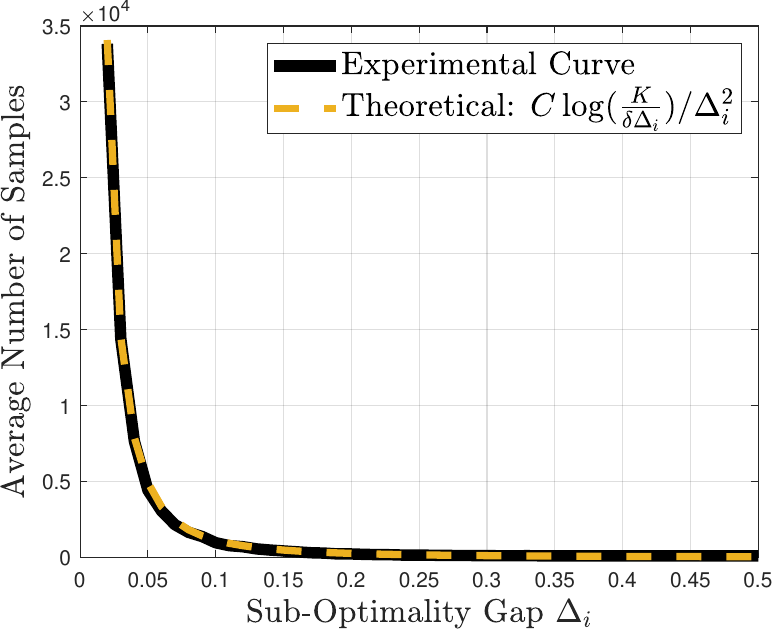}
    \caption{Upper: Gaussian data, $\mu_{i^*}=\mu_{i}=0$, $\sigma_{i^*}=1$ and $\sigma_{i}=0.5$, illustration of the average number of samples (and pulls) as the level of $q$ increases. Lower: Discrete data, $q=0.4$, illustration of the average number of samples (and pulls) as the suboptimality gap $\Delta_i$ increases. In both cases we use $100$ in total independent runs.}
    \label{fig:simulations_matching_bounds}
\end{figure}

\paragraph{Empirical verification of the tightness of the bounds} To empirically validate our theoretical results on the sample complexity, we show in Figure \ref{fig:simulations_matching_bounds} the average number of samples to identify the best arm for a Gaussian (left) and discrete (right) problem setting. For both settings the average number of pulls is evaluated through $100$ independent runs. These curves show that there exists a constant $C$ such that the sample complexity of the algorithm matches our analysis, specifically $C=3/2$ for the left and right figure. For the Gaussian distribution, $\mu_{i^*}=\mu_{i}=0$, $\sigma_i=2$ while $\sigma_{i^*}$ varies. The discrete distribution is provided in Figure \ref{fig:quantile_diff} on the left, $q=0.4$ while the levels $\ell_1,\ell_3$ vary (see Figure \ref{fig:quantile_diff}, left).
\begin{figure}[t!]
    \centering
    \!\!\!\!\includegraphics[width=0.46\textwidth, valign=t]{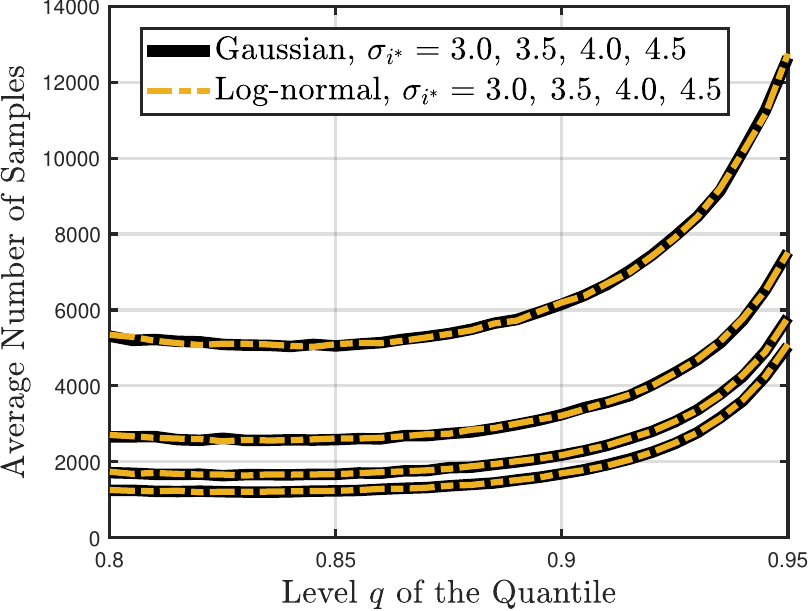}
    \\
    \vspace{11bp}
    \includegraphics[width=0.44\textwidth, valign=t]{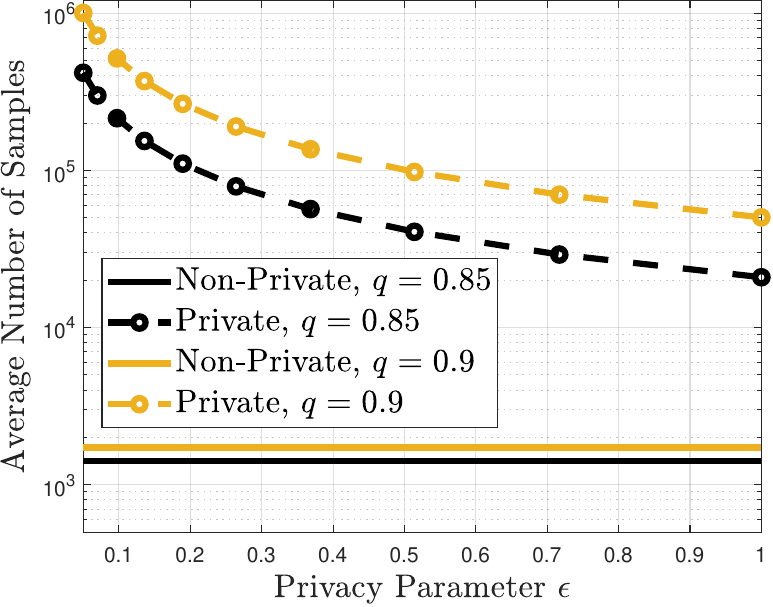}
    \caption{Upper: Sample complexity of Algorithm \ref{alg:succ_el_quant} for Gaussian and log-normal data, $K=2$. For both cases $\mu=0$ for all the arms, $\sigma_i=2$. Lower: Comparison of the private (Algorithm \ref{alg:DP_SE_Quantiles}) and non-private SEQ (Algorithm \ref{alg:succ_el_quant}). We consider log-normal distributions with $K=10$ arms and parameters $\mu_{i^*}=\mu_i=0$, $\sigma_i\in [0.1, 0.2,\ldots,0.9]$ and $\sigma_{i^*}=2$. We compare the estimated number of samples (averaging over 10 iterations) for different values of the privacy parameter $\eps\in[0.05,1]$ and for quantile levels $q=0.85$ and $q=0.9$.}
    \label{fig:simulations}
\end{figure}
\paragraph{Sample complexity for heavy-tailed distributions is not expensive at all} In the following example, we consider two cases (Gaussian and log-normal distributions), for which the differences between the quantile values are different but the gap is identical for any $q\in (0,1)$. The latter can be verified by our Definition \ref{eq:gap_distributions}. As a consequence we expect to find the same average number of pulls for Gaussian and log-normal quantile bandits in our experiment. We can see this by comparing the performance (average termination time averaged over $500$ runs) for a normal distribution and a log-normal distribution for large values of $q$, see Figure \ref{fig:simulations} (left). We take $K = 2$. The suboptimal distribution (normal or log-normal) has mean $0$ and parameter $\sigma = 2$. We vary the best arm by changing $\sigma$. In our definition of gap, the gap between two normal distributions with parameters $\sigma_i$  and $\sigma_{i^*}$ is the same as the gap between two log-normal distributions with parameters $\sigma_i$  and $\sigma_{i^*}$. Each curve shows that the sample complexity when comparing normal and log-normal distributions is the same. In the case of the log-normal distributions the difference in the $q$-quantiles may be quite large. However, the sample complexity of the algorithm depends on the gap.

\paragraph{The cost of privacy for Algorithm \ref{alg:succ_el_quant}} Figure \ref{fig:simulations} (right) shows the performance of Algorithm \ref{alg:succ_el_quant} as a function of the privacy risk $\epsilon$. As expected, as $\epsilon$ increases the sample complexity decreases. The plots show that as the quantile decreases the gap in expected pulls between the private and non-private algorthms decreases. The high cost of privacy in this example shows that there is potential for improvement in the private algorithm: in order to get the sample complexity scaling we chose to double epoch sizes (a standard technique) but empirically we may choose a less aggressive approach. 

\section{Discussion and Future Directions}

In this paper we characterized the sample complexity of the quantile multi-armed bandit problem when the goal is to exactly identify the arm with the highest $q$-quantile in terms of a new measure of suboptimality (gap) between the distributions of each pair of arms. The problem of the lowest $q$-quantile is a simple modification of our method. 
Motivated by scenarios where the arm rewards are private or carry sensitive information, we also provided the first differentially private algorithm for the quantile bandit problem. These privacy considerations lead to an interesting open problem which we discuss next.

\noindent\textit{Open Problem for Privacy.} Algorithm \ref{alg:DP_SE_Quantiles} pulls each suboptimal arm $i$ roughly $K/\eps\Delta_i$ times more than Algorithm \ref{alg:succ_el_quant}. Because we cannot publish approximations of the $q$-quantiles, the factor of $K$ comes because of the need to make private pairwise comparisons. An open question remains: can we avoid this factor of $K$ or is there a converse showing it is necessary? This factor does not appear when looking at the difference between private and non-private best \emph{mean} arm identification. We would like to know if a different elimination procedure would have the same property but for the quantiles.

The bandit literature is vast, with many variations, and for some of these the quantile bandit setting might provide an interesting twist as a form of risk-aware learning. Bandit optimization with risk control is a particularly interesting direction to which this work can apply. For the case of contaminated quantiles~\cite{altschuler2019best} our results imply that $\leq \Delta_i/2$ fraction of contaminated examples could be handled for general $q$-quantiles.
There are still open fundamental questions one may ask, in particular related to the hardness of best-arm-identification for functionals of the distribution beyond the mean and variance~\cite{cassel2018general} in the private and non-private case.


%

\appendices

\section{Proof of Theorem 1} We start by providing a lemma and then we continue with the proof of Theorem 1.\label{app_quantile_flip}

\begin{lemma}
\label{pro:close_dist_close_quantiles}
Let $F$ and $F'$ be two distributions such that $\dtv(F,F')=\eta$. Then for any $q\in (\eta,1-\eta)$ it holds that $F^{-1}(q-\eta) \leq (F')^{-1}(q) \leq F^{-1}(q+\eta)$.
\end{lemma}
\begin{proof}
By definition, for any $x\in \mathbb{R}$ it holds that $|F(x) -F'(x)|\leq \eta$. Define the set $S(F,q) = \{\xi:~ F(\xi)\geq q\}$ where $F^{-1}(q) = \inf S(F,q)$. It follows that any $\xi\in S(F',q)$ also satisfies that $F(\xi) \geq F'(\xi)-\eta \geq q-\eta$ which means $\xi \in S(F,q-\eta)$, and so $F^{-1}(q-\eta) = \inf S(F,q-\eta) \leq (F')^{-1}(q)=\inf S(F',q)$. Similarly, any $\xi \in S(F,q+\eta)$ also belongs to the set $S(F',q)$ proving that $(F')^{-1}(q)\leq F^{-1}(q+\eta)$.
\end{proof}

\label{apx_subsec:proof_lemma_gap_dist_to_flip}
\subsection{Proof of Theorem 1.} First, recall the definition of the distance to flip $\dflip(F_{\rm l} , F_{\rm h} )$ that is equal to 
\begin{align*}
  \inf_{(G_{\rm h} ,G_{\rm l} ): (G_{\rm h} )^{-1}(q) > G_{\rm l} ^{-1}(q)} \max \{ \dtv(F_{\rm l} , G_{\rm h} ), \dtv(F_{\rm h} , G_{\rm l} )\}
\end{align*}
and the definition of the gap 
\begin{align*}
\Delta(F_{\rm l} , F_{\rm h} ) = \sup  \{\eta\geq 0:~ F_{\rm l} ^{-1}(q+\eta) \leq F_{\rm h} ^{-1}(q-\eta) \}.
\end{align*}

Now, given $\eta < \min\{q, 1-q\}/2$ and a distribution $F$ we define two specific shifts. The first is referred to as \emph{$\eta$-push} of $F$ and denoted $F^{\to \eta}$~--- we subtract $\eta$ probability mass 
from the interval $(-\infty, F^{-1}(q))$ and add $\eta$ probability mass to any point or interval in $(F^{-1}(q+2\eta),\infty)$. It is now clear that the $q$-quantile of $F^{\to\eta}$ is in fact $F^{-1}(q+\eta)$ and that $\dtv(F,F^{\to\eta}) = \eta$.
The second shift is equivalent and is an \emph{$\eta$-pull}, denoted $F^{\gets \eta}$~--- we subtract $\eta$-probability mass from the interval $(F^{-1}(q),\infty)$ and move it to the interval $(-\infty, F^{-1}(q-2\eta))$. One can check that the $q$-quantile of $F^{\gets \eta}$ is in fact $F^{-1}(q-\eta)$ and that $\dtv(F,F^{\gets\eta}) = \eta$. 

We now prove the first part of the lemma. Denote that $\Delta(F_{\rm l} ,F_{\rm h} )=\Delta$. Namely, for every $\eta>0$ it holds that $F_{\rm l} ^{-1}(q+\Delta+\eta)>F_{\rm h} ^{-1}(q-\Delta-\eta)$. For any $\eta>0$, consider the $(\Delta+\eta)$-push of $F_{\rm l} $ so that $(F_{\rm l} ^{\to(\Delta+\eta)})^{-1}(q) = F_{\rm l} ^{-1}(q+\Delta+\eta)$ and the $(\Delta+\eta)$-pull of $F_{\rm h} $ so that $(F_{\rm h} ^{\gets (\Delta+\eta)})^{-1}(q) = F_{\rm h} ^{-1}(q-\Delta-\eta)$. Putting these inequalities together shows that $(F_{\rm l} ^{\to(\Delta+\eta)})^{-1}(q) > (F_{\rm h} ^{\gets (\Delta+\eta)})^{-1}(q)$. Applying the definition of the distance to quantile flip, this shows that $\dflip(F_{\rm l} ,F_{\rm h} ) < \Delta+ \eta$ for any positive $\eta$. Thus, $\dflip(F_{\rm l} ,F_{\rm h} ) \leq \Delta$. Specifically, in the case where $\Delta(F_{\rm l} ,F_{\rm h} )=\Delta=0$ we have that $\dflip(F_{\rm l} ,F_{\rm h} )=0$.

We now show the contrapositive. Assume that $\Delta(F_{\rm l} ,F_{\rm h} ) >0$. Fix any $0< \eta <\Delta(F_{\rm l} ,F_{\rm h} )$, and note that it holds that $F_{\rm l} ^{-1}(q+\eta)\leq F_{\rm h} ^{-1}(q-\eta)$. 
Fix any $\tilde G$ and $\tilde H$ such that $\dtv(F_{\rm l} ,\tilde G)\leq \eta$ and $\dtv(F_{\rm h} ,\tilde H)\leq \eta$. It follows from Lemma \ref{pro:close_dist_close_quantiles} and the definition of the gap that
    \begin{align*}
    \tilde G^{-1}(q) \leq F_{\rm l} ^{-1}(q+\eta) 
        \leq F_{\rm h} ^{-1}(q-\eta)
        \leq H^{-1}(q).
    \end{align*}
This shows that any pair of distributions with max TV-distance of $\eta$ to $F_{\rm l} $ and $F_{\rm h} $ is such that that the $q$-quantile has not flipped and it still holds that $\tilde G^{-1}(q)\leq \tilde H^{-1}(q)$, and so $\dflip(F_{\rm l} ,F_{\rm h} )\geq \eta$. Since we've shown that for all $\eta$ that satisfy $\eta<\Delta(F_{\rm l} ,F_{\rm h})$ the inequality $\dflip(F_{\rm l} ,F_{\rm h} )\geq \eta$ holds, it follows that $\dflip(F_{\rm l} ,F_{\rm h}) )\geq \sup_{\eta}\{0< \eta < \Delta(F_{\rm l} ,F_{\rm h})\} = \Delta(F_{\rm l} ,F_{\rm h})$. The last completes the proof. \QEDB

\section{Proof of Properties and Concentration Bound}\label{app:properties}


We start by providing two inequalities for the quantile that we will use later.  
\begin{proposition}\label{eq:quantile_inequalities}
Fix $q\in (0,1)$. Let $F^{-1}_X (q)$ be the $q$-quantile of a random variable $X$ provided in Definition \ref{def:quantile}. Then \begin{align}
    \P \left[ X<F^{-1}_X (q)\right]\leq q
\end{align} and \begin{align}
    \P [X>F^{-1}_X (q)]\leq 1-q.
\end{align}
\end{proposition}
\begin{proof}
Let $x_n$ be a monotonically increasing sequence such that $\lim_{n\to \infty} x_n =F^{-1}_X (q)$, then
	\[  \Pr[X<F^{-1}_X (q)] = \lim_{n\to \infty}\Pr[X\leq x_n] = \lim_{n\to \infty} F(x_n). \]
	 Consider $x_n = F^{-1}_X (q)-2^{-n}$ and assume for sake of contradiction that $\Pr[X<F^{-1}_X (q)]=q+\eps>q$ for some $\eps>0$. It follows that for some $n$ it holds that $F(x_n) > q +\nicefrac{\eps}{2} $. Assuming $\xi=x_n<F^{-1}_X (q)$ in the set $\{ \xi:~ \Pr[X\leq \xi]\geq q   \}$ contradicts the definition of $F^{-1}_X (q)$.
	
	It is true that $\Pr[X > F^{-1}_X (q)] = 1-\Pr[X\leq F^{-1}_X (q)]$. The second part of the claim follows from $\Pr[X\leq F^{-1}_X (q)]\geq q$, the definition of the quantile $F^{-1}_X (q)$ and the fact that the CDF $F_X (\cdot)$ is right continuous.
	\end{proof}



\begin{theorem}[Concentration Bound] \label{thm:GCB2}
Choose a level $q\in (0,1)$. Fix $\delta\in (0,1)$. For any $n\in\mbb{N}$, if \begin{align}
    \sqrt{\frac{  \log(2/\delta)}{2n}} \leq \zeta \leq \min\{q,1-q\}  \label{eq:inequality_conc_11} \end{align} then \begin{align}
   \P\lp F_X^{-1} \lp q\rp\notin \left[X_{(\floor*{n(q- \zeta)})} , X_{(\ceil*{n(q+\zeta)})}  \right] \rp\leq \delta.
\end{align}
\end{theorem}
\begin{proof}
Hoeffding's inequality gives \begin{align}
    &\P \left[ \frac{1}{n}\sum^{n}_{i=1} \nonumber \mathbf{1}_{X_i < F_X^{-1} \lp q\rp }  <  \P [X < F_X^{-1} \lp q\rp] + \sqrt{\frac{  \log(2/\delta)}{2n}} \right]\\& > 1- \delta/2
\end{align} and we have \begin{align*}
    &1-\frac{\delta}{2}\\&<\P \left[ \frac{1}{n}\sum^{n}_{i=1} \mathbf{1}_{X_i < F_X^{-1} \lp q\rp }  <  \P [X < F_X^{-1} \lp q\rp] + \sqrt{\frac{  \log(2/\delta)}{2n}} \right] \\
    & \leq \P \left[ \frac{1}{n}\sum^{n}_{i=1} \mathbf{1}_{X_i < F_X^{-1} \lp q\rp }  <  q + \sqrt{\frac{  \log(2/\delta)}{2n}} \right]\numberthis\label{eq:use_prop}\\
    &= \P \left[ \sum^{n}_{i=1} \mathbf{1}_{X_i < F_X^{-1} \lp q\rp }  <  n\lp q +\sqrt{\frac{  \log(2/\delta)}{2n}}\rp \right]\\
    &\leq \P \left[ \sum^{n}_{i=1} \mathbf{1}_{X_i < F_X^{-1} \lp q\rp }  <  \ceil*{ n \lp q +\sqrt{\frac{  \log(2/\delta)}{2n}}\rp} \right]\\
    &\leq \P \left[ \sum^{n}_{i=1} \mathbf{1}_{X_i < F_X^{-1} \lp q\rp }  < \ceil*{ n \lp q +\zeta \rp} \right],  \forall \zeta\geq \sqrt{\frac{  \log(2/\delta)}{2n}},\nonumber
\end{align*} and \eqref{eq:use_prop} comes from Proposition \ref{eq:quantile_inequalities}. The last inequality implies that for all $\zeta \geq \sqrt{  \log(2/\delta)/2n}$ the following holds \begin{align}\label{eq:bound_prob_empirical_CDF}
    \P \left[\sum^{n}_{i=1} \mathbf{1}_{X_i < F_X^{-1} \lp q\rp }  \geq  \ceil*{ n \lp q +\zeta \rp} \right]\leq \delta/2.
\end{align} It is true that \begin{align} \nonumber
    &X_{(\ceil*{n(q+\zeta)})} < F_X^{-1} \lp q\rp \\&\iff \sum^n_{i=1} \mathbf{1}_{X_i < F_X^{-1} \lp q\rp}\geq \ceil*{n(q+\zeta)},\label{eq:claim1}
\end{align} and \eqref{eq:bound_prob_empirical_CDF} gives \begin{align}\label{eq:First_part}
     \P \left[ X_{(\ceil*{n(q+\zeta)})}< F_X^{-1} \lp q\rp \right]\leq \delta/2,\, \forall \zeta\geq \sqrt{\frac{  \log(2/\delta)}{2n}}.
\end{align}

\noindent Similarly, for all $\zeta \in \left[\sqrt{  \log(2/\delta)/2n}, q \right]$
\begin{align*}
    &1-\frac{\delta}{2}\numberthis \label{eq:bound_prob_emprical_CDF4}\\
    &\leq  \P \left[ \sum^{n}_{i=1} \mathbf{1}_{X_i > F_X^{-1}\lp q \rp  }  < n- \floor*{n \lp q-\zeta \rp } \right]\\
    &\leq \P \left[ \sum^{n}_{i=1} \mathbf{1}_{X_i > F_X^{-1}\lp q \rp  }  < n-\max\{\floor*{n \lp q-\zeta \rp },1\}+1 \right],
\end{align*} 
and due to the definition of order statistics with restricted indices in the set $\{1,2,\ldots,n\}$ it is true that
\begin{align}\nonumber
    &X_{(\floor*{n(q-\zeta)})} > F_X^{-1} \lp q\rp\\& \iff \sum^{n}_{i=1} \mathbf{1}_{X_i > F_X^{-1}\lp q \rp}\geq n-\max\{\floor*{n \lp q-\zeta \rp },1\}+1.
\end{align}
Then \eqref{eq:bound_prob_emprical_CDF4} gives \begin{align} \label{eq:Second_part}
    \P \left[   X_{(\floor*{n(q-\zeta)})} > F_X^{-1} \lp q\rp \right]\leq \delta/2,
\end{align}
for all $\zeta \in \left[\sqrt{  \log(2/\delta)/2n}, q \right]$. Finally, \eqref{eq:First_part}, \eqref{eq:Second_part} and the union bound give the statement of the lemma.
\end{proof}

\section{Analysis of the Cases $\Delta_i>0$ and $\Delta_i=0$}\label{appendix_Gap}
The next corollary provides an analysis for the cases of strictly positive or zero gap. Note that there exists $A>0$ such that $F^{-1}_{i^{*}}(q-A)\geq F^{-1}_{i}(q+A)$ if and only if $\Delta_i>0$. Conversely, it does not exist $A>0$ such that $F^{-1}_{i^{*}}(q-A)\geq F^{-1}_{i}(q+A)$ if and only if $\Delta_i=0$. The last two statements are direct consequence of the definition of the gap $\Delta_i$ (see Definition \ref{dap_definition}).
\begin{corollary}\label{corollary_cases}
Define \begin{align}
    Q_{R,F_i}(q) &\triangleq \inf \{x : F_i (x)>q \},\\
    L^{-}_{F_{i^{*}}}(q)&\triangleq \max_{x}\{F_{i^*}(x) :F_{i^*}(x)<q  \}.
\end{align}
Assume that the best arm is unique,\begin{align}
     F^{-1}_{i}(q) <   F^{-1}_{i^*}(q),\quad \forall i\in \{1,2\ldots,K\}\setminus \{i^*\}.
\end{align}
\begin{enumerate}
\item If the CDF $F_i(\cdot)$ is continuous at $Q_{R,F_i}(q)$ it is always true that $Q_{R,F_i}(q) < F^{-1}_{i^*}(q)$, 
and there exists $A>0$ such that $F^{-1}_{i^{*}}(q-A)>F^{-1}_{i}(q+A)$.

\item If the CDF $F_i(\cdot)$ is not continuous at $Q_{R,F_i}(q)$ then there are three sub-cases.
\begin{itemize}
    \item $Q_{R,F_i}(q) < F^{-1}_{i^*}(q)$: There exists $A>0$ such that \begin{align}\label{eq:fir_case}
        F^{-1}_{i^{*}}(q-A)>F^{-1}_{i}(q+A).
    \end{align} The values of $A$ that satisfy \eqref{eq:fir_case} are\begin{align}
        A(\alpha,x) = \alpha \min\{ F_i (x) -q , q- F_{i^*} (x) \},
    \end{align} for any $\alpha\in (0,1)$ and $x\in[Q_{R,F_i}(q),F^{-1}_{i^*}(q))$.
    
    \item $Q_{R,F_i}(q) = F^{-1}_{i^*}(q)$: There does not exist $A>0$ such that   $F^{-1}_{i^{*}}(q-A)>F^{-1}_{i}(q+A)$. 
    
    \begin{itemize} 
    \item If $F_{i^*}(\cdot)$ has a discontinuity at $F^{-1}_{i^*}(q)$ there exists $A>0$ such that \begin{align}
    F^{-1}_{i^{*}}(q-A)=F^{-1}_{i}(q+A).\label{eq:sec_case_I}
\end{align}
The values of $A$ that satisfy \eqref{eq:sec_case_I} are \begin{align*}
    A(\eps)&=\min \{F_i(Q_{R,F_i}(q))-q, q-L^{-}_{F_{i^{*}}}(q)-\eps \}\\
    &= \min \{F_i(F^{-1}_{i^*}(q))-q, q-L^{-}_{F_{i^{*}}}(q)-\eps \}
\end{align*}
for any $\epsilon\in(0,q-L^{-}_{F_{i^{*}}}(q)]$.

\item If $F_{i^*}(\cdot)$ is continuous at $F^{-1}_{i^*}(q)$ then there does not exist $A>0$ such that \begin{align}
    F^{-1}_{i^{*}}(q-A)\geq F^{-1}_{i}(q+A).
\end{align} 
\end{itemize}
     
    \item $Q_{R,F_i}(q) > F^{-1}_{i^*}(q)$: It does not exist $A>0$ such that \begin{align}
    F^{-1}_{i^{*}}(q-A)\geq F^{-1}_{i}(q+A).
\end{align} 
\end{itemize}
\end{enumerate}
\end{corollary}


\subsection{Proof of the case $Q_{R,F_i}(q) < F^{-1}_{i^*}(q)$ in Corollary \ref{corollary_cases}}
We show that 
    \begin{align}\label{eq:positive_gap}
    \Delta_i \triangleq  \min \{ F_i (x) -q , q- F_{i^*} (x) \} >0 ,
    \end{align} for all $x\in [Q_{R,F_i}(q),F^{-1}_{i^*}(q))\subset [F^{-1}_{i}(q),F^{-1}_{i^*}(q)]$
and 
    \begin{align}
    F_{i^*}^{-1} (q-\alpha \Delta_i) > F^{-1}_{i}(q+\alpha \Delta_i),
    \end{align} for all $x\in [Q_{R,F_i}(q),F^{-1}_{i^*}(q))$ and $\alpha\in (0,1)$.
To show \eqref{eq:positive_gap} it is sufficient to find the minimum value $\tilde{x}$ in the interval $[F^{-1}_{i}(q),F^{-1}_{i^*}(q)]$ such that \begin{align}
    F_i (\tilde{x}) -q>0 \text{ and } q- F_{i^*} (\tilde{x}) >0.
\end{align} Notice that \begin{align}
    q- F_{i^*} (x) >0,\quad \forall x\in [F^{-1}_{i}(q),F^{-1}_{i^*}(q)),
\end{align} thus we have to find the minimum value $\tilde{x}$ in the interval $[F^{-1}_{i}(q),F^{-1}_{i^*}(q))$ such that \begin{align}
    F_i (\tilde{x}) -q>0.
\end{align} By the assumption $Q_{R,F_i}(q) < F^{-1}_{i^*}(q)$ it follows that \begin{align}
     Q_{R,F_i}(q) \triangleq \inf \{x : F_i (x)>q \} < F^{-1}_{i^*}(q).
\end{align} This implies that 
    \begin{align}
    \tilde{x}\equiv  Q_{R,F_i}(q).
    \end{align} 
Further \begin{align}
   F^{-1}_{i}(q)\leq Q_{R,F_i}(q)<F^{-1}_{i^*}(q),
\end{align} which implies that 
    \begin{align}
    \Delta_i \triangleq  \min \{ F_i (x) -q , q- F_{i^*} (x) \} >0,
    \end{align} for all $x\in [Q_{R,F_i}(q),F^{-1}_{i^*}(q))\subset [F^{-1}_{i}(q),F^{-1}_{i^*}(q)] $.

For any $x\in [Q_{R,F_i}(q),F^{-1}_{i^*}(q))$ and any $\alpha\in (0,1)$ it is true that
\begin{align} \label{eq: long_ineq_1}
    F_{i^*} \lp F_{i^*}^{-1} (q-\alpha \Delta_i) \rp \geq q-\alpha \Delta_i > q- \Delta_i  \geq F_{i^*}(x),
\end{align} where the last inequality comes from the definition of $\Delta_i$. As a consequence of \eqref{eq: long_ineq_1} \begin{align}
     F_{i^*} \lp F_{i^*}^{-1} (q-\alpha \Delta_i) \rp &>  F_{i^*}(x) \implies\nonumber\\
     F_{i^*}^{-1} (q-\alpha \Delta_i) &> x \label{eq:long_ineq_2}
\end{align} because $F_{i^*}(\cdot)$ is increasing. Further, for any $x\in S$ and any $\alpha\in (0,1)$ it is true that \begin{align} \label{eq:long_ineq_3}
    x\geq F^{-1}_{i}(F_i (x)) \geq F^{-1}_{i}(q+\Delta_i) \geq F^{-1}_{i}(q+\alpha \Delta_i)
\end{align}because $F_i(x)\geq \Delta_i+q$ (that comes from the definition of the $\Delta_i$) and $F_i^{-1}(\cdot)$ is increasing. Now \eqref{eq:long_ineq_2} and \eqref{eq:long_ineq_3} give  \begin{align}
    F_{i^*}^{-1} (q-\alpha \Delta_i) > F^{-1}_{i}(q+\alpha \Delta_i), 
\end{align}for any $x\in [Q_{R,F_i}(q),F^{-1}_{i^*}(q))$ and $\alpha\in (0,1)$.






\subsection{Proof of the case $Q_{R,F_i}(q) = F^{-1}_{i^*}(q)$ in Corollary \ref{corollary_cases}}  




First we show that if $Q_{R,F_i}(q) = F^{-1}_{i^*}(q)$ and the best arm is unique $F_i^{-1}(q)<F^{-1}_{i^*}(q)$, then there does not exist $A>0$ such that \begin{align}
    F^{-1}_{i^{*}}(q-A)>F^{-1}_{i}(q+A). \label{eq:impossible_inequality}
\end{align} 
We use contradiction to show that there does not exist $A>0$ such that $F^{-1}_{i}(q+A)=F^{-1}_{i}(q)$. Assume that for some $A>0$ \begin{align*}
    F^{-1}_{i}(q) &=  F_i^{-1}(q+A)\implies \\
F_i\lp  F^{-1}_{i}(q)\rp &=    F_i\lp F_i^{-1}(q+A)\rp \geq q+A \implies\\
    F_i (F_i^{-1}(q))&\geq q+A\implies\\
    F_i (F_i^{-1}(q))&\geq F_i ( Q_{R,F_i}(q)),\numberthis 
\end{align*} where the last line cannot hold only as a strict inequality $F_i (F_i^{-1}(q))> F_i ( Q_{R,F_i}(q))$ because of the monotonicity of $F_i(\cdot)$. Additionally, the definition $Q_{R,F_i}(q) \triangleq \inf \{x : F_i (x)>q \}$ gives that $F_i (F_i^{-1}(q))= F_i ( Q_{R,F_i}(q)) $ if and only if $Q_{R,F_i}(q) = F_{i}^{-1}(q)$\footnote{The level $q$ is not in the codomain of $F_i(\cdot)$}. The latter does not hold because $Q_{R,F_i}(q) = F_{i^*}^{-1}(q)> F_{i}^{-1}(q)$. Combining the above we get the contradiction. As a consequence for every $A>0$ it is true that \begin{align*}
     &F_i^{-1}(q+A) >  F_i^{-1}(q) \\&\implies F_i^{-1}(q+A) \geq Q_{R,F_i}(q)\\
     &\implies  F_i^{-1}(q+A) \geq F_{i^*}^{-1}(q)\numberthis \label{eq:claim2_inequality1}\\
     &\implies  F_i^{-1}(q+A) \geq F_{i^*}^{-1}(q-A)\numberthis.
\end{align*} The last line completes the statement of \eqref{eq:impossible_inequality}, and the inequality $F^{-1}_{i^{*}}(q-A)>F^{-1}_{i}(q+A)$ holds only for $A=0$. 

If the CDF $F_{i^*}(\cdot)$ has a discontinuity at $F^{-1}_{i^*}(q)$ there exists $A>0$ such that \begin{align}
    F^{-1}_{i^{*}}(q-A)=F^{-1}_{i}(q+A).
\end{align} 
Define $L^{-}_{F_{i^{*}}}(q)\triangleq \max_{x}\{F_{i^*}(x) :F_{i^*}(x)<q  \} $. and recall that $Q_{R,F_i}(q)=F^{-1}_{i^{*}}(q)$ then \begin{align*}
    A_1 &\triangleq \sup \{A:\! F_i^{-1}(q+A)=Q_{R,F_i}(q) \}\!=\!F_i(Q_{R,F_i}(q))-q\\
    \mc{A}_2 &\triangleq  \{A:  F_{i^*}^{-1}(q-A)=F^{-1}_{i^{*}}(q) \}\\&=\{ q-L^{-}_{F_{i^{*}}}(q)-\eps: \eps\in(0,q-L^{-}_{F_{i^{*}}}(q)]\}.
\end{align*} For any $\epsilon\in(0,q-L^{-}_{F_{i^{*}}}(q)]$ define $A_2 (\eps)= q-L^{-}_{F_{i^{*}}}(q)-\eps \in \mc{A}_2$, then the quantity \begin{align}
    A^*(\eps) \triangleq \min \{A_1, A_2 (\eps)\}
\end{align} satisfies the condition  \begin{align}
    F^{-1}_{i^{*}}(q-A^*(\eps))=F^{-1}_{i}(q+A^{*}(\eps)).
\end{align} On the other hand if $F_{i^*}(\cdot)$ is continuous at $F^{-1}_{i^*}(q)$ then  for every $A>0$ \begin{align}
    F^{-1}_{i^*}(q-A)<F^{-1}_{i^*}(q).
\end{align} The latter combined with the inequality \eqref{eq:claim2_inequality1} give that for every $A>0$ it is true that $F^{-1}_{i^*}(q+A)> F^{-1}_{i^*}(q-A)$. As a consequence there does not exist $A>0$ such that 
    \begin{align}
    F^{-1}_{i^{*}}(q-A)\geq F^{-1}_{i}(q+A).
    \end{align}

\subsection{Proof of the case $Q_{R,F_i}(q) > F^{-1}_{i^*}(q)$ in Corollary \ref{corollary_cases}} 
    We use contradiction to show that there does not exist $A>0$ such that $F^{-1}_{i}(q+A)=F^{-1}_{i}(q)$. Assume that for some $A>0$ \begin{align*}
    F^{-1}_{i}(q) &=  F_i^{-1}(q+A)\implies \\
F_i\lp  F^{-1}_{i}(q)\rp &=    F_i\lp F_i^{-1}(q+A)\rp \geq q+A \implies\\
    F_i (F_i^{-1}(q))&\geq q+A\implies\\
    F_i (F_i^{-1}(q))&\geq F_i ( Q_{R,F_i}(q)), 
   \numberthis
\end{align*} where the last line cannot hold only as a strict inequality $F_i (F_i^{-1}(q))> F_i ( Q_{R,F_i}(q))$ because of the monotonicity of $F_i(\cdot)$. Additionally, the definition $Q_{R,F_i}(q) \triangleq \inf \{x : F_i (x)>q \}$ gives that $F_i (F_i^{-1}(q))= F_i ( Q_{R,F_i}(q)) $ if and only if $Q_{R,F_i}(q) = F_{i}^{-1}(q)$\footnote{The level $q$ is not in the codomain of $F_i(\cdot)$}. The latter does not hold because $Q_{R,F_i}(q) = F_{i^*}^{-1}(q)> F_{i}^{-1}(q)$. Combining the above we get the contradiction. For every $A>0$ it is true that \begin{align*}
     &F_i^{-1}(q+A) >  F_i^{-1}(q) \\
     &\implies F_i^{-1}(q+A) \geq Q_{R,F_i}(q)\\
     &\implies  F_i^{-1}(q+A) > F_{i^*}^{-1}(q)\\
     &\implies  F_i^{-1}(q+A) > F_{i^*}^{-1}(q-A).\numberthis
\end{align*} As a consequence, there does not exist $A>0$ such that \begin{align}
        F^{-1}_{i^{*}}(q-A)\geq F^{-1}_{i}(q+A).
    \end{align}



\ifCLASSOPTIONcaptionsoff
  \newpage
\fi



\bibliographystyle{IEEEtran}
\bibliography{mab}

%

%








\end{document}